\documentclass[12pt]{amsproc} 

\usepackage[margin=1in]{geometry}
\usepackage[utf8]{inputenc}
\usepackage[dvipsnames]{xcolor}
\usepackage[shortlabels]{enumitem}
\usepackage{graphicx}
\usepackage{amsmath,amssymb,amsthm,mathtools}
\usepackage{MnSymbol}
\usepackage{mathrsfs}
\usepackage{latexsym, verbatim}
\usepackage{appendix}
\usepackage{units}
\usepackage{amsmath}
\usepackage{dsfont}
\usepackage[normalem]{ulem}
\usepackage[hidelinks]{hyperref}
\usepackage{algorithm}
\usepackage{algpseudocode}
\usepackage{soul}
\usepackage{nicematrix}
\usepackage{pgfplots}
\pgfplotsset{compat=1.18}
\usetikzlibrary{pgfplots.groupplots}
\usepackage[numbers,sort]{natbib}
\usepackage{caption}
\mathtoolsset{showonlyrefs}

\usepackage{listings}
\usepackage{xcolor}
\mathtoolsset{showonlyrefs}
\newtheorem{theorem}{Theorem}[section]
 
\newtheorem{lemma}[theorem]{Lemma}
\newtheorem{proposition}[theorem]{Proposition}

\theoremstyle{definition}

\newtheorem{definition}[theorem]{Definition}

\newtheorem*{theorem**}{Theorem\theoremnum}
\newenvironment{theorem*}[1][]{%
  \edef\theoremnum{\if\relax\detokenize{#1}\relax\else~#1\fi}
  \begin{theorem**}
}{%
  \end{theorem**}
}  
\usepackage{xr}

\newcommand{\fa}{\mathbf{a}}
\newcommand{\fb}{\mathbf{b}}
\newcommand{\fc}{\mathbf{c}}
\newcommand{\fx}{\mathbf{x}}
\newcommand{\fz}{\mathbf{z}}
\newcommand{\fv}{\mathbf{v}}

\newcommand{\Acal}{\mathcal{A}}

\newcommand{\Ocal}{\mathcal{O}}

\newcommand{\RR}{\mathbb{R}}

\newcommand{\Ncal}{\mathcal{N}}

\DeclareMathOperator{\Vect}{Vect}
\DeclareMathOperator{\Cov}{Cov}
\DeclareMathOperator{\Diag}{Diag}

\DeclareMathOperator{\tr}{tr}
\DeclareUnicodeCharacter{2060}{\nolinebreak}
\newcommand{\T}{^\mathsf{T}}

\newcommand{\SB}{\mathbb{S}}
\newcommand{\fw}{\mathbf{w}}

\title{Multi-context principal component analysis}

\author[K. Wang et al.]{Kexin Wang$^*$}
\address{Kexin Wang, Harvard University}
\email{kexin\_wang@g.harvard.edu}
\thanks{$^*$ denotes equal contribution. }
\author[]{Salil Bhate$^{*, \dagger}$}
\address{Salil Bhate, Broad Institute of MIT and Harvard}
\email{salil92@gmail.com}
\author[]{Jo\~ao M. Pereira}
\address{Jo\~ao M. Pereira, University of Georgia}
\email{jpereira@uga.edu}
\author[]{Joe Kileel}
\address{Joe Kileel, University of Texas at Austin}
\email{jkileel@math.utexas.edu}
\author[]{Matylda Figlerowicz}
\address{Matylda Figlerowicz, Harvard University}
\email{matyldafiglerowicz@g.harvard.edu}
\author[]{Anna Seigal$^\dagger$}
\address{Anna Seigal, Harvard University}
\email{aseigal@seas.harvard.edu}
\thanks{$^\dagger$ to whom correspondence should be addressed.}

\date{}

\makeatletter
\let\AMS@printaddresses\@setaddresses
\let\@setaddresses\@empty
\makeatother
\begin{document}

\begin{abstract}
   Principal component analysis (PCA) is a tool to capture factors that explain variation in data. 
   Across domains, data are now collected across 
   multiple contexts (for example, individuals with different diseases, cells of different types, or words across texts). While the factors explaining variation in data are undoubtedly shared across subsets of contexts, no tools currently exist to systematically recover such factors.
   We develop multi-context principal component analysis (MCPCA), a theoretical and algorithmic framework that 
   decomposes data into factors shared across subsets of contexts. Applied to gene expression, MCPCA reveals axes of variation shared across subsets of cancer types and an axis whose variability in tumor cells, but not mean, is associated with lung cancer progression. Applied to contextualized word embeddings from language models, 
   MCPCA maps stages of a debate on human nature, revealing a discussion between science and fiction over decades.
   These axes are not found by combining data across contexts or by restricting to individual contexts. 
   MCPCA is a principled generalization of PCA to address
   the challenge of understanding factors underlying data across contexts.
\end{abstract}

\maketitle

\section{Introduction}

 Principal component analysis (PCA) is a universal tool in data analysis. It finds directions along which variance is maximized in a dataset and, in doing so, minimizes the reconstruction error of the data. Principal components (PCs)
	 are fundamental factors that explain data, as a standalone tool and in machine learning pipelines~\cite{jolliffe2016principal,novembre2008genes,yeung2001principal,nandi2015principal,xu2012outlier,dhal2022comprehensive}.

Technological advances now produce data collected across multiple, related contexts. Examples include 
patient samples across disease types~\cite{chang2015impact}, 
single-cell measurements under perturbations or from different cell states \cite{meinshausen2016methods,dixit2016perturb,rao2021exploring,lonsdale2013genotype}, and literature across genres and time periods ~\cite{gao2020pile,mnih2013learning}. 
One has a fixed set of variables (such as genes), which are measured across multiple contexts. 
The number of contexts ranges from two or three (e.g., patient groups) to thousands (e.g., genetic perturbations). 

We propose a new modeling principle that extends PCA to the multi-context setting. We call it multi-context PCA (MCPCA). Just as usual PCA decomposes data into axes that explain variation in a dataset, MCPCA decomposes multi-context data into axes that explain variation in subsets of contexts. We show that the multi-context principal components (MCPCs) are factors explaining the data within and across contexts that cannot be found via alternative methods.

The approach of MCPCA, to find axes that optimally explain variation in subsets of contexts, is a new idea. It builds on two threads of work. The first is to find axes common to multiple contexts~\cite{flury1983some,flury1984common,flury1987two,flury1986algorithm,ponnapalli2011higher,sentana2001identification} 
and the second is to compare factors between two contexts~\cite{alter2003generalized,abid2018exploring,wang2024contrastive}.
   Finding factors common across contexts originated in work of Flury in the 1980s~\cite{flury1984common}. He considered a set of variables measured in multiple contexts and formulated the \emph{hypothesis of common principal components}, the modeling assumption that the covariance matrices share the same eigenvectors. 
(His example had six variables and two contexts~\cite{flury1983angewandte}.)  
A similar idea appears in econometrics, where the covariance matrices of different volatility regimes share common structure~\cite{sentana2001identification,rigobon2003identification,lewis2021identifying}. 
Finding factors that explain the difference between two contexts also has a long history, starting with canonical correlation analysis~\cite{hotelling1936relations} and the generalized singular value decomposition~\cite{alter2003generalized,van1976generalizing}. 

Operationally, MCPCA generalizes the diagonalization of a single covariance  matrix (as in PCA) to the simultaneous diagonalization of a collection of matrices -- the covariance matrices of each context. 
Simultaneous diagonalization is well-studied~\cite{ziehe2004fast,cardoso1996jacobi,bunse1993numerical}. 
However, existing algorithms are too slow or inaccurate for large real-world datasets. 
We propose a new algorithm and demonstrate its superiority in terms of accuracy and speed. 
It is based on a recent advance in tensor decomposition~\cite{wang2025multi}, applied to a stack of covariance matrices, with new partial non-negativity conditions imposed (since variance is non-negative).

The outputs of MCPCA are the MCPCs and the \emph{context loadings} for each MCPC. The MCPCs are directions explaining variance in the subsets of contexts indicated by its context loadings. MCPCA does not impose restrictions on the context loadings: it does not require specifying in advance which factors are shared or individual, which subset of contexts they appear in, or how the factors change from one context to another. It requires no threshold for comparing factors and its only hyperparameter is the rank, just like usual PCA.  It generalizes finding common factors (those with high or comparable context loading in every context) and finding factors with high relative importance in one context versus another. 

We demonstrate qualitatively and quantitatively that MCPCs and their context loadings capture factors across contexts in 
synthetic data and two real-world 
domains: gene expression profiles in biology and contextualized word embeddings in fiction and scientific works.

In biology, MCPCA reveals factors of gene expression whose patterns across contexts, as given by their context loadings, capture crucial structure to relate contexts. We apply MCPCA to gene expression profiles of patients across different cancer types from The Cancer Genome Atlas. It decomposes the data into organ-specific, multi-cancer, and pan-cancer axes of variation. These capture biologically meaningful hallmarks shared across subsets of cancers. We find an axis of variation that relates activation of metal ion transport and activation of extracellular matrix remodeling  genes, predominantly present in thyroid cancer samples, which defines a subset of pancreatic cancer patients with improved survival. Applied to single-cell gene expression profiles across lung cancer patients from the Lung Cancer Atlas, we find an axis of variation with  hypoxia/stress response at one pole and oxidative phosphorylation/proliferation at the other whose context loadings and variance explained -- but not average position -- is predictive of stage. This axis cannot be found by PCA on the combined data. We benchmark the context loadings for gene-function annotation from Perturb-seq data and for phylogenetic reconstruction from single-cell gene expression of orthologous genes.

Applied to contextualized word embeddings, MCPCA provides a quantitative way to map the dynamics of intellectual debates of word meanings across time periods and genres of cultural and scientific production. We apply MCPCA to contextualized  embeddings of the word `human' in sentences from Project Gutenberg  measured by language models. The contexts are textual forms (Science vs. Fiction) across time periods (1800-1820, 1820-1840, 1840-1860, 1860-1880, 1880-1900). 
The works under science include texts from disciplines of science and some works of philosophy and social science.
The MCPCs reveal two factors present in both science and fiction, though at different times. 
A factor representing questioning of the essence of humanity is present in fiction from 1800-1820, is then addressed by science texts between 1820-1840, and subsequently answered in science texts by
a factor representing a debate between humans as species and humans as citizens initiated by
the \emph{On the Origin of Species}, and then this debate is picked back up in fiction in 1880-1900. These factors cannot be found by comparing factors within contexts or across the combined data.

MCPCA is a principled algorithm for addressing a fundamental, unresolved challenge in modern data analysis. It highlights the insights that can be gleaned across domains by finding axes that optimally explain variance shared across subsets of contexts.

\section{Results}

\subsection{A model for factors shared across subsets of contexts}

MCPCA is a method for variables (such as genes) that are measured through data assigned to contexts (such as tumor types).
The assignments of data to context (e.g. the type of a tumor) are assumed to be known. 
The modeling assumption of MCPCA is 
the \emph{hypothesis of collective components}: factors that explain variation in the data are shared across subsets of contexts.  See \textbf{Figure~1A}. 
MCPCA recovers these factors and their importance in each context, without specification of which contexts share factors or how their importance changes across contexts.  

MCPCA models the covariance matrix in the $i$th context as $\Sigma_i = A B_i A\T$, where $A$ is the matrix of MCPCs, which are shared across contexts, and the non-negative diagonal matrix $B_i$ gives the weight of each MCPC in the $i$th context. 
The weights of the MCPCs across contexts are collected into a single matrix $B$, called the  \emph{context loading matrix}, whose $i$th row gives the weight of each MCPC in the $i$th context (equivalently, whose $j$th column gives the weight of the $j$th MCPC in each context). See \textbf{Figure 1B} and Section~\ref{sec:input_output_model}. 
The columns of $A$ are scaled to have unit norm, since they represent directions.

The output of MCPCA is the matrix of MCPCs $A$ and the context loading matrix $B$. They are both useful in downstream analysis. Matrix $A$ is analogous to the matrix of PCs in usual PCA: the MCPCs are factors of interest in the data, useful for scientific hypothesis generation and testing. 
The context loadings matrix encodes each context by the importance of the MCPCs.
We show that it is superior to alternatives (such as the variance along usual PCs in that context) for prediction tasks.

An alternative viewpoint on MCPCA is a latent variable model. Observed variables are a transformation and weighting of 
shared latent variables~$\mathbf{z}$ that are uncorrelated and unit variance.  
The latent variables and transformation are shared across contexts and the weighting is context specific. 
That is, $\mathbf{x}_i = A B_i^{1/2} \mathbf{z}$, where $\mathbf{x}_i$ are the observed variables in the $i$th context and matrices $A$ and $B_i$ are as above. The exponent $1/2$ is because scaling a variable by $\beta^{1/2}$ scales its variance by $\beta$. 
Not all latent variables appear in all contexts,  since diagonal entries of $B$ may be zero.
Note that MCPCA, like PCA, seeks to model the variation in the data, and both ignore information about the mean.
To project 
to the space defined by the MCPCs, one uses the pseudo-inverse $A^\dagger$, which is called the \emph{MCPC projection matrix}. It projects to the MCPC space for linearly independent MCPCs, since $A^\dagger A = I$.

\begin{figure}[htbp]
\centering
\includegraphics[scale = 0.5]{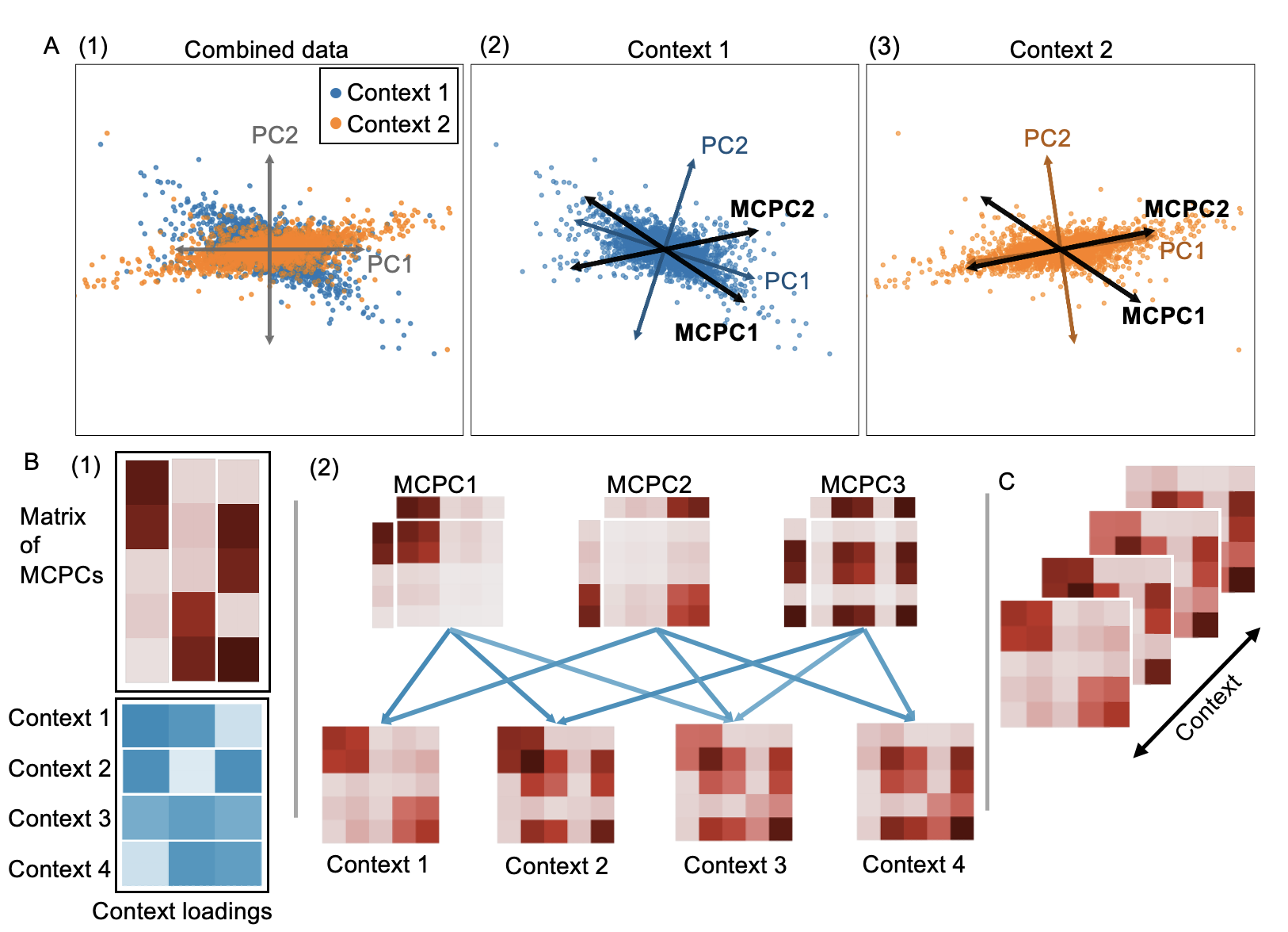}
\captionsetup{width=0.95\textwidth}
\caption{\textbf{A model for multi-context principal components.}
\newline \textbf{A.} (1) Data assigned to known contexts (here, blue and orange) plotted along the top two overall PCs, (2) and (3) Data in individual contexts, plotted in the space of overall PCs, the within-context PCs marked in blue and orange, respectively, and MCPCs shown in black. The MCPCs are non-orthogonal directions that capture axes of variation shared across contexts.
\newline \textbf{B.} (1) the parameters in MCPC: the matrix of MCPCs (reds) and the context loading matrix (blues). This example has five variables, four contexts, and three MCPCs.  (2) The MCPCs are components of covariance. Each is weighted by the context loading to approximate the covariance matrices in each context. 
\newline \textbf{C.} The covariance tensor is obtained by stacking the covariance matrices from each context.
}
\label{fig1}
\end{figure}

\subsection{Tensor decomposition estimates multi-context principal components}
Just as usual PCs are the vectors in a low-rank approximation of a covariance matrix, the MCPCs are the vectors in a low-rank approximation of the \emph{covariance tensor}, the tensor obtained by stacking together the covariance matrices of the different contexts (\textbf{Figure 1C}).
Given $p$ variables and $k$ contexts, the covariance tensor has size $p \times p \times k$. 
MCPCA finds a low-rank approximation of this tensor, see Section~\ref{sec:reconstruction}. 
This discovers collective components without pairwise comparison of contexts, via the connection between tensor decomposition and simultaneous diagonalization~\cite{de2006link}.
Unlike previous work on tensor decomposition for data analysis, see e.g.~\cite{kolda2009tensor},
it does not require a pairing of samples from different contexts (for example, a pancreatic tumor need not be paired with a corresponding thyroid tumor).
 
Tensor decomposition establishes the theoretical foundation of MCPCA. It has a unique solution, which (unlike PCA) does not require orthogonality of the MCPCs, see Section~\ref{sec:identifiability}. 
MCPCA generalizes the four main facets of PCA:  minimizing reconstruction error, maximizing variance explained, transforming to uncorrelated variables, and estimating parameters in a multivariate Gaussian model, see Sections~\ref{sec:reconstruction} to \ref{sec:mle}.

\subsection{An accurate and scalable algorithm for estimating MCPCs}

A priori, any tensor decomposition algorithm can be used for MCPCA: it is a rank $r$ approximation of the covariance tensor. 
We design an algorithm for MCPCA based on a partially non-negative adaptation of the multi-subspace power method (MSPM)~\cite{wang2025multi}. The algorithm uses subspace projection and deflation, followed by non-negative least squares. 

We conduct experiments with synthetic and semi-synthetic data to compare to other tensor decomposition algorithms, in terms of accuracy, scalability, and sample complexity.
We generate synthetic data by sampling from an $r$-dimensional Gaussian reweighted via the matrix $B$ and transformed via the MCPCs in $A$, see \textbf{Figure 2A}. 
We sample from these distributions, compute the sample covariance matrices, and decompose them via various matrix and tensor decomposition algorithms. 
We record the accuracy and time taken for MCPCA and the 12 other methods in  \textbf{Figure 2B}.
MCPCA is the fastest of the three accurate methods. The other two were FFDIAG~\cite{ziehe2004fast} and QRJ1D~\cite{afsari2006simple}, which were 30 and 300 times slower, respectively. 
Nonlinear least squares (NLS)~\cite{vervliet2016tensorlab} was accurate on average, though sensitive to initialization. We further discuss existing algorithms in Section~\ref{sec:comparison}.
We quantify dependence on sample size in \textbf{Figure 2C}.
MCPCA exhibits the best sample complexity
of the methods.
We give theoretical error analysis in Section~\ref{sec:error}.

MCPCA does not impose orthogonality of the MCPCs. When they are orthogonal, they coincide with usual PCs, which can be found using PCA of the individual contexts or of the combined data. 
We show the advantages of MCPCA even in the orthogonal case, where it has numerical and sample complexity advantages over matrix-based methods such as PCA, in \textbf{Supplementary Figure S2}.
Intuition for this benefits of MCPCA over matrix methods is that the ill-posed set (the inputs for which the method has non-unique output) is smaller for a tensor method than for a matrix method, see Section~\ref{sec:orthogonal}.

We qualitatively validated MCPCA with a semi-synthetic example. We constructed three datasets of MNIST digits \cite{deng2012mnist} superimposed onto a background of clouds or grass from ImageNet \cite{deng2009imagenet}. 
One context contains only the background, one contains digits~0 and~2 superimposed onto background images, and the third contains digits~1 and~2, superimposed onto the background images, see \textbf{Figure 2D}. 
The digits appear in subsets of contexts. The MCPCs are images, since each loading is a pixel intensity. MCPCA finds the digits and correct context loadings, see \textbf{Figure 2E}. 
Though sparsity is not hard coded into MCPCA, in practice the context loadings have few large loadings. 
Alternative tensor decomposition algorithms (except for NLS) and PCA did not find interpretable output, see \textbf{Figure 2F}.

\begin{figure}
    \centering
    \includegraphics[width=0.8\linewidth]{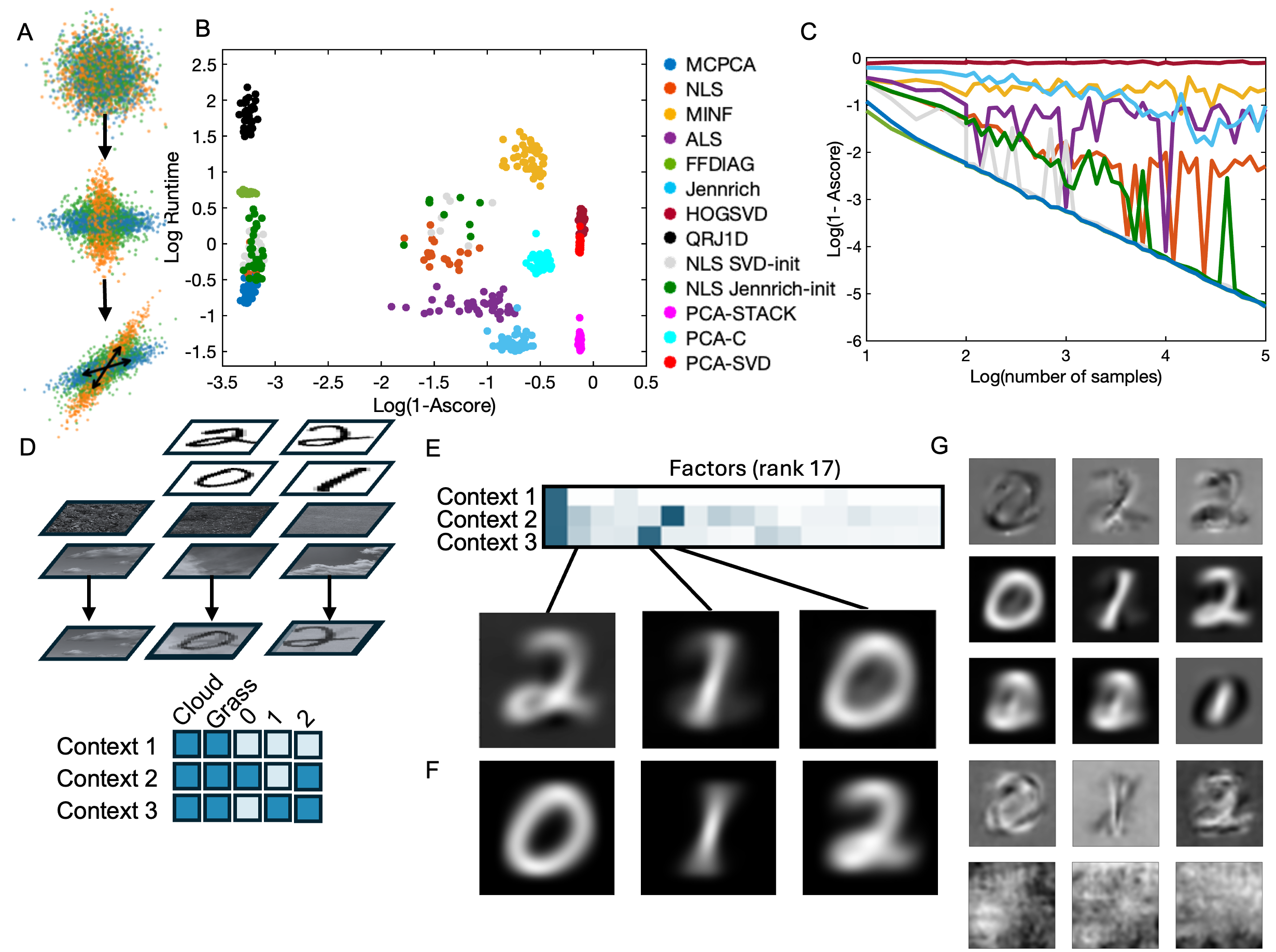}
    \captionsetup{width=0.95\textwidth}
    \caption{\textbf{⁠Benchmarking MCPCA with synthetic and semi-synthetic data.}
\\ \textbf{A.} Synthetic data generated from a standard multivariate Gaussian by transforming by a context-dependent diagonal scaling followed by a shared linear transformation. MCPCs marked as arrows in the transformed data. 
\\ \textbf{B.} Runtime (in seconds) against accuracy (Ascore) for the 13 methods in the legend. MCPCA lower left, the fastest accurate method.
\\ \textbf{C.} Accuracy compared to sample size for the 13 methods listed in the legend. 
\\ \textbf{D.} Semi-synthetic data schematic: data generated via background images of clouds and grass with digits superimposed. Ground truth context loadings shown in $3 \times 5$ array.  
\\ \textbf{E.} Three MCPCs plotted via pixel intensities. Context loading matrix as heatmap. Digit 0 has high loading in context 2, digit 1 has high loading in context 3 and digit two has non-zero loadings in contexts 2 and 3. This matches ground truth in D.
\\ \textbf{F.} The four ground truth averaged digits (zero in context 2, two in context 2, one in context 3, two in context 3).
\\ \textbf{G.} Digits recovered via FFDIAG, NLS, PCA-STACK, QRJ1D, HO GSVD.
}
    \label{fig:placeholder}
\end{figure}

\subsection{MCPCs of cancer}

We apply MCPCA to RNA sequencing data. The input are data matrices with rows labeled by samples and columns by genes; the matrix entry is the normalized log-transformed count of transcripts in the cell. 

In molecular biology, the PCs of gene expression are the axes in gene expression space that together optimally explain variance. These reflect the biological processes driving phenotypic heterogeneity (alongside technical artifacts).  Given multi-context gene expression data, MCPCs are the axes that optimally explain variance across contexts. We therefore hypothesized that MCPCA applied to multi-context gene expression data could reveal the biological processes driving heterogeneity across  phenotypically distinctive contexts.

We first considered  The Cancer Genome Atlas, which has gene expression measurements across 30 cancer types. It was previously shown that gene expression similarity is driven predominantly by cell-of-origin: samples from different cancer types do not cluster together~\cite{hoadley2018cell}. However, the processes driving phenotypic variability in different cancer types presumably overlap, even though the cancer types are phenotypically distinctive. This should reflect, for example, a balance between inflammatory and immunosuppressive tumors present in multiple cancer types. 

We applied MCPCA to gene expression measurements from 10509 tumor samples across the 30 cancer types. We performed PCA on gene expression  across all samples, and constructed the covariance tensor, each slice representing the PC-PC (i.e. PC $\times$ PC) covariance matrix within individual cancer types (\textbf{Figure 3A}). 
We determined the optimal rank (30) for MCPCA by examining stability of MCPCs and context loadings across rank, and consistency of MCPCs with 10\% of the data held out  (\textbf{Supplementary Figure S3A-C}). We interpreted the learned MCPCs with gene set enrichment  analysis using Enrichr \cite{kuleshov2016enrichr}.

MCPCA provided a biologically coherent decomposition of cancer phenotypic heterogeneity into organ-specific, multi-cancer, and pan-cancer axes of variation (\textbf{Figure 3B}, \textbf{Supplementary Figure S3D-E} and \textbf{Supplementary Table S1}). Organ-specific axes correctly captured biological processes specific to the organs they were detected in, including MCPC21 (metabolic processes present in liver), MCPC29 (renal transport processes in the kidney) and MCPC16 (T cell activation specific to thymoma).  Pan cancer axes captured hallmark balances present in many tumors: MCPC0 (retinoid metabolism vs. angiogenesis), MCPC1 (aerobic respiration),  MCPC2 (Mitosis/proliferaiton vs. immune signaling) and MCPC3 (chemokine signaling/inflammation vs. cholesterol biosynthesis). Multi-cancer axes captured processes common across organ systems, for example MCPC11 (cilium assembly vs. amino acid tranposrt shared in uterine corpous carcinoma and lung adenocarcinoma), and MCPC4 (epidermal differentiation present in skin and squamous cell cancers).

We found an MCPC that was a major axis of variation in thyroid carcinomas, but defined a subgroup of pancreatic adenocarcinoma tumors with improved survival, see \textbf{Figure 3B, arrow}. MCPC10 was a multi-cancer MCPC that reflected a balance between metal ion transport and antigen presentation/extracellular matrix organization. It had a high context loading only in thyroid carcinoma and pancreatic adenocarcinoma and did not explain variance in other adenocarcinomas, see  \textbf{Figure 3C}. Inspecting cancer samples with respect to the overall PCs, we see that this MCPC indeed defines a major axis of variation in thyroid cancer (\textbf{Figure 3D}) and a subgroup of pancreatic cancer patients (\textbf{Figure 3E}). This subgroup is associated with improved survival in pancreatic cancer \textbf{Figure 3F}. Performing PCA on pancreatic cancers individually, we see that MCPC10 represents the direction of separation between two groups of patients.

MCPCA decomposes phenotypic variability into components shared across subsets of cancers, and leverages coherence in gene expression variation across cancer types to find biologically meaningful axes of variation that are missed when combining data or looking at cancer types individually. Our results emphasize the clinical relevance of cancer gene expression module pleiotropy across tumor types.

\begin{figure}[htbp]
    \centering
    \includegraphics[width=0.95\linewidth]{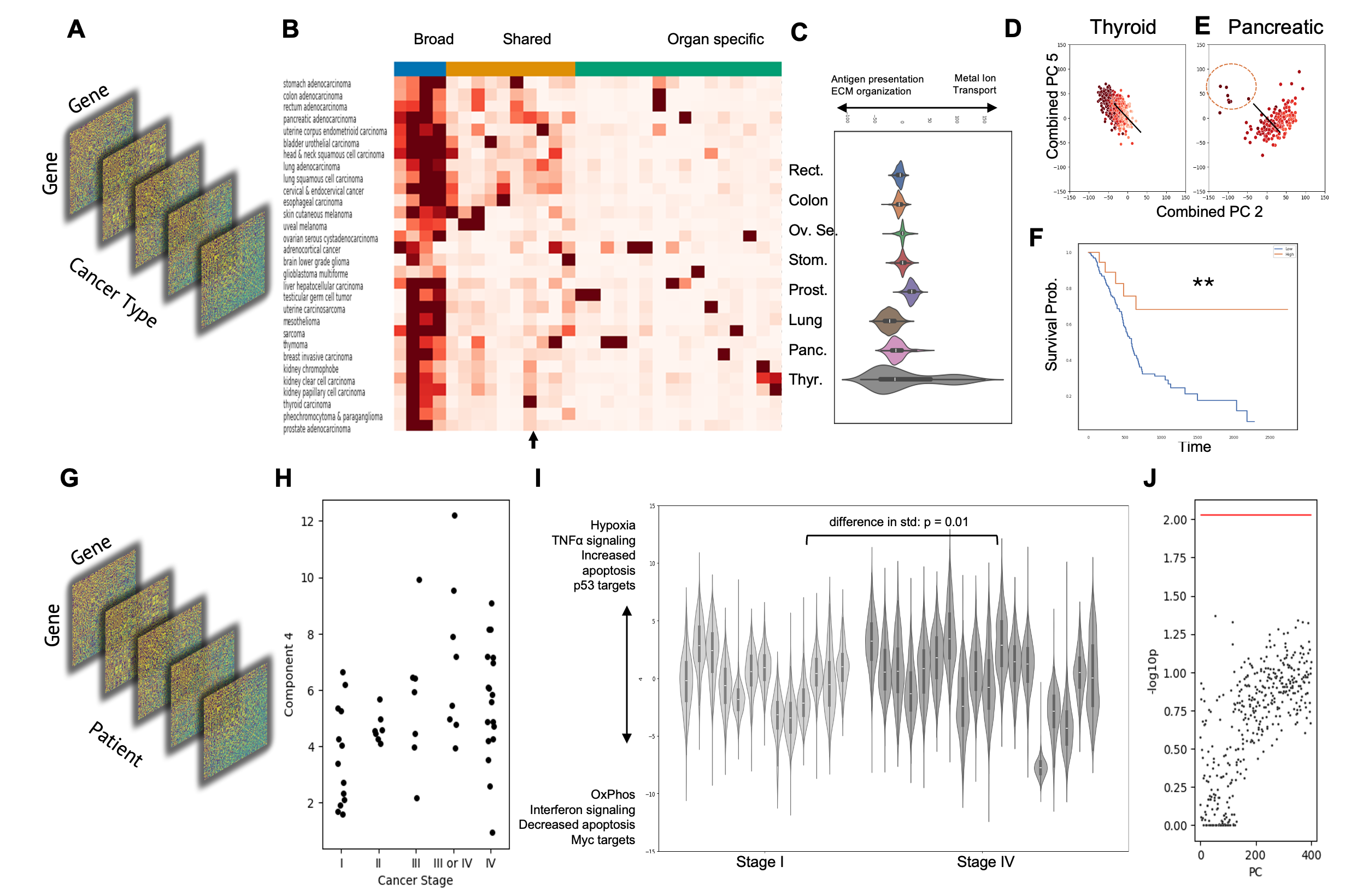}
    \captionsetup{width=0.95\textwidth}
    \caption{\textbf{Multi-context principal components of cancer across tumors and patients.}
\newline \textbf{A.} Gene-gene covariance tensor across cancer types.
\newline \textbf{B.} Context loadings from MCPC with rank 30. Colors indicate grouping into broad, multicancer and organ-specific MCPCs
\newline \textbf{C.} Distribution of scores of MCPC10 in patients with cancers of highest context loading. 
\newline \textbf{D.} Scores of overall PCs 2 and 5 (those with highest similarity to MCPC10) in thyroid cancer. Color indicates MCPC10 score. Black line is projection of MCPC10.
\newline \textbf{E.} Scores of overall PCs 2 and 5 in pancreatic cancer. Color indicates MCPC10 score. Black line indicates projection of MCPC10 on PCs 2 and 5. Dashed circle indicates patients with highest score along MCPC10.
\newline \textbf{F.} Kaplan-Meier plot of patients with high and low MCPC10. $**$ indicates Cox Proportional-Hazards regression p-value $= 0.001$.
\newline \textbf{G.} Schematic of gene-gene covariance tensor across lung cancer patients.
\newline \textbf{H.} Distribution of context loadings of MCPC5 per patient grouped by cancer stage.
\newline \textbf{I.} Distribution of  per-cell MCPC5 score in Stage I and IV patients; t-test p-value for difference in std of scores $=0.01$, p-value for difference in mean not significant.
\newline \textbf{J.} -Log10pvalues of association with cancer stage (maximum p-value for t-test difference between groups I/II vs III/IV and Spearman correlation with stage) for std of overall PCs 1-400. Red line indicates significance obtained by MCPC5.
}
    \label{fig3}
\end{figure}

\subsection{Variation in single-cell gene expression heterogeneity across patients}
A current challenge in single-cell analysis is to model inter-individual variability in cell state heterogeneity. This is important in cancer, where tumor cell heterogeneity is indicative of plasticity and resistance potential \cite{yuan2019cellular}. MCPCA finds axes that explain variance across multiple patients and therefore provides a systematic way to quantify single-cell heterogeneity shared across subsets of individuals. We investigated the MCPCs of gene-gene covariance matrices across patients from the lung cancer atlas \cite{salcher2022high}. Here, each patient is a context and each cell is a sample.

We applied MCPC with five components to gene-gene covariance matrices of tumor cells from  lung adenocarcinoma atlas patients, performing PCA on the combined data of all samples to reduce to 400 overall components (\textbf{Figure 3G}).  Restricting to tumor cells means axes of variation are driven by cell state diversity, not differences in cell type composition.

The  context loading of MCPC5 was predictive of cancer stage (\textbf{Figure 3H}, Bonferroni $p < 0.05$ with Spearman correlation and t-test); patients with later stage cancer had a higher context loading. We interpreted MCPC5 by identifying gene sets enriched in the positive and negative directions of the axis. (MCPCs, just like PCs, are only unique up to sign, so positive and negative here indicate only two poles.)

One direction of MCPC5 indicated a stressed, hypoxic state characterized by enrichment of  TNFa signaling,  apoptosis pathway and p53 target gene sets. The other indicated an oxygen rich, proliferative state characterized by oxidative phosporylation,  proliferation and target gene sets Myc (\textbf{Supplementary Table S2}). Crucially, a patient's mean score along the axis (i.e., the average lean towards hypoxic/stress response or oxygen rich/proliferative)  was not predictive of stage, but the variance  was (\textbf{Figure 3I}).

This axis was stable to resampling single cells, and across random seeds. It was not correlated with the proportion of metastatic cells in the sample (\textbf{Supplementary Figure S3F}). While the mean expression of PC6 was also predictive of stage, the context loading for MCPC5 was not correlated with it (\textbf{Supplementary Figure S3G}), and the context loading was significant in a multivariate model including PCs (\textbf{Supplementary Figure S3H}) indicating that it was an independent predictor. Furthermore, none of the top 50 PCs had a standard deviation (std) predictive of stage, so this multi-context axis is emergent from analysis across contexts, see \textbf{Figure 3J}.

Thus, MCPCA identifies a plasticity hallmark of lung cancer progression: tumor cell diversity along a hypoxia/stress/p53 - oxygen rich/proliferative/Myc axis, but not average tumor cell positioning along it, increases with cancer progression.

\subsection{Benchmarking context-level representations for gene expression from MCPCA}
We benchmarked whether context loadings from MCPCA better captured information in gene-gene covariance matrices.

We first considered the task of reconstructing phylogeny using human, chimpanzee, gorilla, rhesus macaque, and marmoset brain single-cell gene expression, a task that requires contrasting axes of variation in cellular populations. We performed an initial dimensionality reduction with PCA applied to all orthologs, and then compared whether hierarchically clustering species by their context loadings in MCPCA, or by baseline feature sets, recovered the correct phylogenetic tree (\textbf{Supplementary Figure S4A-I}). 

MCPCA with rank greater than 8 recovers the correct tree (\textbf{Supplementary Figure S4B-C}). 
The mean of each species' cells in PC space does not recover the correct tree, indicating covariation information is required to reconstruct phylogeny (\textbf{Supplementary Figure S4D-E}).  
Using vectorizations of the PC-PC covariance matrices recovers the correct tree only with 33 or more components (\textbf{Supplementary Figure S4F-G}). Variance explained by each PCs does not recover the correct tree (\textbf{Supplementary Figure S4H-I}). 

A current benchmarking task is to identify functional gene properties from single-cell gene expression profiling of gene perturbation effects with Perturb-seq. The contexts are gene perturbations. Vector representations of contexts are evaluated by whether context similarity (which corresponds to gene-gene relationships in the case of Perturb-seq) enriches known, gene-gene functional annotations \cite{bendidi2024benchmarking}. We constructed a covariance tensor of gene $\times$ gene $\times$ perturbation using the  Perturb-seq data of~\cite{replogle2022mapping}, performed MCPCA, and evaluated whether concatenating MCPCs to average PC in each context improves recall of known gene relationships. Across multiple gene sets, including cell-type markers, gene ontology terms, and GWAS hits, MCPC improves recall over mean PC as well as mean PC concatenated with variance explained by PCs (\textbf{Supplementary Figure S4J}). 

Thus MCPCA provides an efficient way to extract evolutionary information by contrasting  gene-gene covariance matrices of orthologous genes from different species and to extract functional information by  contrasting gene-gene covariance across perturbation effects.

\begin{figure}[htbp]
\centering
\includegraphics[width = 0.9\linewidth]{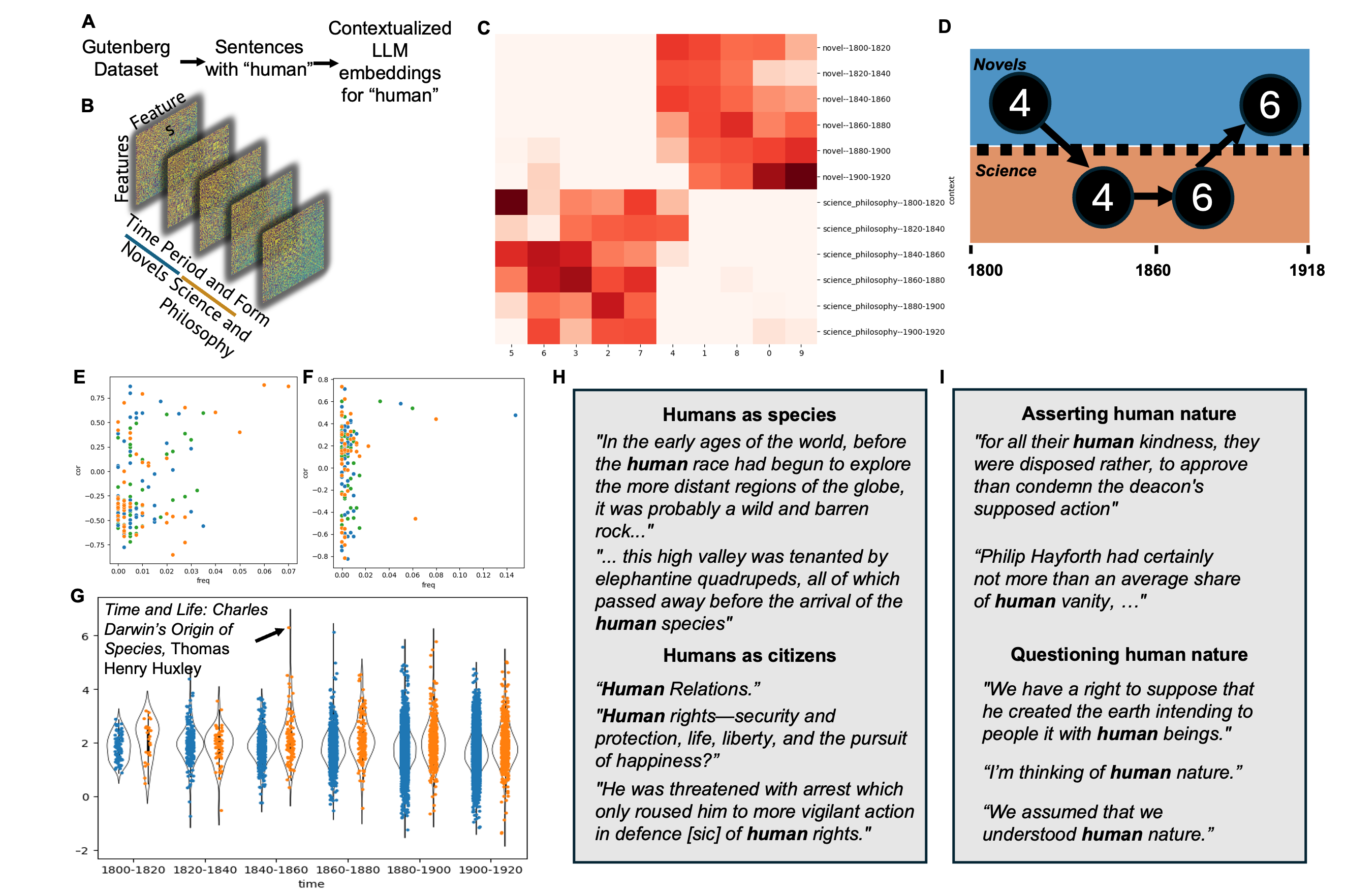}
\captionsetup{width=0.95\textwidth}
\caption{
\textbf{ ⁠Multi-context principal components of semantics across literary forms and time periods.}
\newline \textbf{A.} Pipeline for feature extraction from contextualized embeddings of the word “human” in the Gutenberg dataset.
\newline \textbf{B.} Input to MCPC. Covariance matrices of human embeddings across literary forms (science and fiction) and time intervals (20 year intervals between 1800 and 1900).
\newline \textbf{C.} Context loadings of MCPCs across contexts.
\newline \textbf{D.} Graph summarizing temporal ordering of factors observed in both literary forms by their context loadings.
\newline \textbf{E.} Correlation of context variances ($y$ axis) and overlap between top 200 sentences ($x$ axis) with factors obtained by PCA in individual contexts and overall PCs for MCPC4.
\newline \textbf{F.} Correlation of context variances ($y$ axis) and overlap between top 200 sentences ($x$ axis) with factors obtained by PCA in individual contexts and overall PCs for MCPC6.
\newline \textbf{G.} Distributions of median scores of MCPC6 across works grouped by time and form. 
\newline \textbf{H.} Annotation of MCPC6 using sentences with highest and lowest scores.
\newline \textbf{I.} Annotation of MCPC4 using sentences with highest and lowest scores. 
}
\label{fig4}
\end{figure}

\subsection{Examining intellectual history across time periods and genres of textual production}

Contextualized word embeddings from large language models (LLMs) of a word in a sentence capture its relational meaning \cite{reif2019visualizing}; PCs of contextualized word embeddings of a single important word across a corpus capture relational axes of variability - which we refer to as debates - around that concept \cite{kozlowski2019geometry}. We therefore hypothesized that MCPCs of word embeddings of a fixed concept across time and literary form could offer quantitative insights to trace the emergence and transformation of debates in intellectual history that unfold across disciplines.

We obtained contextualized word embeddings of the word `human' in the DeepMind Gutenberg Dataset \cite{raecompressive2019} using the BERT language model \cite{devlinbert2018} (\textbf{Figure 4A}). We grouped sentences by time period (twenty year periods 1800-1920) and form (science or fiction), and constructed covariance matrices between embeddings in each context (\textbf{Figure~4B}). We applied MCPCA with ten factors. We chose the word `human' because there were key changes in its usage during this time period.

Most MCPCs were shared across time periods but restricted to one form (science or fiction) (\textbf{Figure 4C}). However, two MCPCs were present across both science and fiction: MCPC4 and MCPC6. We constructed a graph summarizing the dynamics of these two MCPCs across time and form (\textbf{Figure 4D}): the context loadings of MCPC4 were high in fiction prior to in science; MCPC6 had high loading initially in science after the peak of MCPC4, and was subsequently observed in fiction. Applying PCA to the embeddings in individual forms, or collectively, did not produce factors that recapitulated these (measured with respect to cosine similarity, correlation across time periods/forms, or overlap in top sentences) (\textbf{Figure 4E-F}).

We grouped works by their median score for MCPC6, across sentences containing `human'. We observed a peak for one work in 1859 (\textbf{Figure 4G}). This was \emph{Time and Life} by Thomas Henry Huxley, a magazine article discussing Charles Darwin's \emph{On the Origin of Species}, published the preceding month. This high score seems to aptly reflect the cultural relevance of \emph{On the Origin of Species} and Huxley's role in popularizing Darwin's theory.

Investigating the top sentences at each pole of MCPC6, we saw that one direction corresponded to the use of the term human as natural objects positioned alongside geography and geology. The other direction positioned humans alongside relations, rights and other social relationships. Thus, MCPC6 reveals how the \emph{On the Origin of Species} ignited a debate in fiction between humans as biological versus social objects (\textbf{Figure 4H, Supplementary Table S3}).
MCPC4 captured a debate between questioning versus assuming human nature. At one pole was a positioning of human nature as unknown and, at the other, assertions of routine aspects of human nature (\textbf{Figure 4I, Supplementary Table S3}).

Interpreting these MCPCs with their temporal dynamics encoded by the context loadings (\textbf{Figure 4D}) we obtain a  model for how fiction was a site of questioning the existence and character of `human nature'; this conversation was addressed by scientific discourses via a debate in terms of humans as species versus citizens, which in turn was picked back up by fictional texts. Thus, the analysis objective of MCPCA -- contrasting axes of variation across subsets of contexts -- can be applied to the internal representations of language models to glean quantitative insights about data at scale.

\section{Discussion}

We introduced MCPCA, a conceptual and computational framework to find axes that explain variation across subsets of contexts. We formulated the problem via partially symmetric tensor decomposition, proving that it generalizes the properties that define PCA  to the multi-context setting (including minimizing reconstruction error, maximizing variance explained, transforming data to uncorrelated variables, and estimating parameters in a multivariate Gaussian model). 

We provided an algorithm that uses subspace projection and deflation to identify MCPCs. Through benchmarking on synthetic and semi-synthetic data, we showed that the approach outperforms other methods and tensor decomposition approaches in terms of accuracy, sample complexity, and scalability. 

We show the trans-disciplinary relevance of MCPCA via two application domains (gene expression and intellectual history). The second studied data from contextualized word embeddings, showing how MCPCA is a tool to use on top of large language models. 
MCPCA served as a principled technique to reveal the compositional structure of contexts, present across disciplines and useful within disciplines. It identified meaningful factors that could not be seen via alternative methods. 

 Gene expression studies have highlighted the importance of understanding  gene expression hetereogeneity across cancer types \cite{hoadley2018cell,kinker2020pan,gavish2023hallmarks}. These studies use clustering or non-negative matrix factorization followed by heuristic merging of factors across cancer types. In contrast, applied to patient gene expression samples from different cancer types as contexts, MCPCA provides a decomposition of biological processes across distinct cancers into organ-specific, shared, and pan-cancer gene expression modules. Underscoring the clinical significance of pan-cancer MCPCs, an MCPC defined a patient group in pancreatic cancer with improved survival by integrating heterogeneity present in thyroid cancer.

We applied MCPCA to single-cell gene expression in lung cancer cells, with individual patients as contexts. Unlike previous methods to perform sample-level dimensionality reduction \cite{liu2025learning,boyeau2025deep}, which learn nonlinear embeddings useful for comparing similarity or fitting predictive models, MCPCA directly identifies the dominant variability in cellular heterogeneity across samples in a principled and interpretable manner. Specifically, we found that cellular diversity, but not cellular position, along a hypoxia/stress response versus oxidative phosphorylation/proliferation axis, was a hallmark of cancer progression.  These axes could not be found by merging data across contexts or by investigating individual contexts. We further showed in quantitative benchmarking that MCPCA outperforms PCA for extracting evolutionary relationships from single-cell gene expression data across species, and for extracting gene function information from cellular heterogeneity.

The multi-context factor estimation enabled by MCPCA applies to the output of machine learning models. We applied it to LLM embeddings of the word `human' across texts, split by form (science vs. fiction) and by twenty year period (between 1800-1900) to map debates and their transmission across forms and time. We identified debates that could not be identified by combining data across contexts or by viewing contexts individually. Interpreting the MCPCs, we saw an axis questioning the essence of humans arise in fictional texts, preceding an axis answering it with humans as species versus humans as citizens arising in science, initiated by \emph{On the Origin of Species}, before being discussed in fictional texts in the 1880s.

This analysis offers tools for considering genres of textual production jointly and across a large volume of materials. It underscores the necessity for this type of study, showing that debates often treated as specific to one scientific discipline or genre of cultural production are cross-pollinated, and respond to shared historico-political conditions. This opens avenues for further research, such as a detailed examination of genres through which cultural concepts are interrogated and propagated. For example, in the presented data, literature for children and young adults are strongly represented among the highest-scoring works (\textbf{Supplementary Table S4}, works in bold).

As PCA has inspired a class of machine learning algorithms~\cite{scholkopf1998nonlinear,tenenbaum2000global,hinton2006reducing}, we anticipate that MCPCA will also serve as a useful toolkit in deep learning models.  Key limitations are the choice of rank and scalability. Here, we have provided a CPU implementation, but anticipate that a GPU implementation will provide speedups.

\section{Methods}

\subsection{Tensor decomposition}

Given $p$ variables in $k$ contexts, for each context $i = 1, \ldots, k$ we form the $p \times p$ sample covariance matrix $S_i$.
The covariance tensor $T$ is their stack, a tensor of size $p\times p\times k$.
The coupled decomposition $S_i = A B_i A\T$ is equivalent to the decomposition of the covariance tensor as $T = \sum_{j=1}^r \fa_j \otimes \fa_j \otimes \fb_j$, where $\fa_j$ and $\fb_j$ are the $j$th columns of matrices $A$ and $B$, respectively. The columns of $A$ are scaled to have unit norm.  

The MCPCA algorithm adapts MSPM~\cite{wang2025multi} to incorporate partially non-negative output. 
The vectors $\fb_j$ are non-negative, since variance is non-negative and $(\fb_j)_i$ is the contribution to the variance in context $i$ coming from direction $\fa_j$.
The MCPCA algorithm transforms the covariance tensor to a smaller tensor of size $p \times k \times r$ using the singular value decompositon (SVD), uses power method iterations~\cite{de1995higher} to find the MCPCs (the matrix $A$), and then finds the context loadings (matrix $B$) via non-negative least squares~\cite{paatero1997weighted}. 

A Python implementation of MCPCA is available, along with installation instructions, at \url{https://github.com/QWE123665/MCPCA}.
Given as input a rank and list $X_{\rm list} = (X_1,\ldots,X_k)$ of NumPy arrays, each with the same number of columns (i.e., one data matrix per context), the MCPCs and context loadings are computed as follows:
\begin{lstlisting}
from mcpca import MCPCA 
model = mcpca(n_components = r)
model.fit(X_list)
A = model.components_
B = model.loadings_
\end{lstlisting}

\subsection{Choosing the rank}\label{sec: choose rank}

A key challenge in any dimensionality reduction tool is to choose the dimension, which here (as in usual PCA) is the rank $r$. 
The rank is a hyperparameter that controls the tradeoff between parsimony and variance explained. The rank is selected based on the stability of the algorithm across runs, to ensure that the components reflect intrinsic structure in the data.

In the noiseless setting, the rank of $T$ is the rank of the matrix obtained by reshaping it into a $p \times pk$ matrix, provided $T$ is sufficiently general of rank at most~$p$. 
Motivated by this, we use the scree plot of its singular values to find a range of candidate ranks, analogously to PCA.
In this range, we quantify the stability of each rank by measuring the similarity between the MCPC matrices obtained from repeated runs under random initialization.
Unless otherwise specified, we choose the rank in MCPCA by choosing the biggest rank for which the MCPCs are consistent across runs, balancing expressivity with robustness.

Rank selection from a set of candidate ranks can be performed via
\begin{lstlisting}
from mcpca import MCPCA
model = MCPCA(n_components=None, rank_range=rank_list, n_seed_pairs=5, cos_sim_threshold=0.8)
r = model.r_
\end{lstlisting}
where $\rm{rank}\_\rm{list}$ is a list of candidate ranks (e.g., suggested by a scree plot). The parameter \verb|n_seed_pairs| specifies how many pairs of runs with random initialization are used to assess stability. For each candidate rank, we compute the average cosine similarity score across these  pairs of runs, and select the largest rank whose average score exceeds \verb|cos_sim_threshold|.

\subsection{Synthetic analysis}\label{sec: synthetic}

We compare MCPCA to 12 other methods, in terms of accuracy, time taken, and sample complexity:
\begin{enumerate}
    \item PCA and eigendecomposition methods ($\times 4$): PCA on the combined data (PCA-STACK), PCA for each dataset followed by clustering PCs into $r$ groups and taking the cluster means as the shared axes (PCA-C), PCA in each context followed by SVD on the PCs, taking the top $r$ right singular vectors as the shared axes (PCA-SVD), the higher-order generalized singular value decomposition (HO GSVD) \cite{ponnapalli2011higher}. 

        \item Simultaneous diagonalization methods ($\times 3$):
    FFDIAG~\cite{ziehe2004fast}, QRJ1D~\cite{afsari2006simple} (for $A$ non-orthogonal), Jacobi~\cite{cardoso1996jacobi} (for $A$ orthogonal), Jennrich~\cite{harshman1970foundations} 

    \item Tensor decomposition algorithms ($\times 5)$: 
    nonlinear least squares (NLS) (vanilla version, SVD initialization, and Jennrich initialization), unconstrained nonlinear optimization (MINF), alternating least squares (ALS), as implemented in Tensorlab.

\end{enumerate}
We focus on recovery of the MCPCs as not all methods can compute context loadings.  
We match the recovered and true components ($\fa_i'$ and $\fa_i$, respectively) via a greedy approach: we match $\fa_1$ with the vector among the $\fa_i'$ with which it has the largest absolute cosine similarity, flip its sign to make the cosine similarity positive, then proceed to $\fa_2$ with the remaining columns, and so on until all are matched.  
We then compute the mean cosine similarity between the matched pairs
$$
\text{Ascore}(\{\fa_i\}_{i=1}^r, \{\fa_i'\}_{i=1}^r) = \frac{1}{r}\sum_{i=1}^r |\langle \fa_i, \fa_i'\rangle|.
$$

In the experiments, we let $k = 50$, $p = 100$, and $r = 60$. We generate matrices
$A \in \RR^{100 \times 60}$ and $B \in \RR^{50 \times 60}$ each $40$ times.
The matrix $A$ is obtained by normalizing the columns of a random
Gaussian matrix, while $B$ is constructed as a sparse nonnegative
matrix 
by having non-zero entries present with density 0.2 and sampling the non-zero entries from the absolute values of a sample from $\Ncal(0,1)$.  
Each time, we sample $1000$ observations from each of the $50$
multivariate Gaussian distributions $\Ncal(0, A B_i A\T)$, where $B_i$ is the $60 \times 60$ diagonal matrix formed from
the $i$th row of $B$.
The covariance tensor has size $100 \times 100 \times 50$ and we compute its best rank $60$ approximation. 
We compared the recovered MCPCs  with the true MCPCs using the Ascore.
We evaluated accuracy and runtime across methods.
Figure~1B plots the log-runtime against $\log(1-\mathrm{Ascore})$ for all
algorithms for the $40$ trials.

In the second experiment, we fix
$A \in \RR^{100 \times 60}$ and $B \in \RR^{50 \times 60}$ and vary the sample
size in each context from $10$ to $10^5$, comparing accuracy against sample size. We exclude QRJ1D due to its prohibitive computational cost.

\subsection{Semi-synthetic analysis}

We superimpose hand-written digits 0, 1 and 2 from MNIST \cite{deng2012mnist} onto grass and cloud images from \cite{deng2009imagenet}.  All images have size $28\times 28$ -- they are pixel intensities for $784$ pixels. 
The first context consists of 5000 cloud images and 5000 grass images. 
The second context has 8000 grass and 2000 cloud images.
We sample 10000 images of digits 0 and 10000 images of digits 2 and superimpose them on the grass and cloud images with independent strengths, via the uniform distribution $\text{Uniform}[0,1]$. 
The third context has 2000 grass and 8000 cloud images.
Next, we sample 10000 images of digit $1$ and 10000 images of digit $2$ and superimpose them on the grass and cloud images, again with strength following $\text{Uniform}[0,1]$.
Samples are shown in \textbf{Supplementary Figure S2D}. 

The covariance tensor has size $784 \times 784 \times 3$, and we compute a rank-$r$ approximation with $r=17$. The choice of rank follows the procedure in Section \ref{sec: choose rank}; it is the largest rank such that the average Ascore of five pairs of random runs exceeds 0.8, see  \textbf{Supplementary Figure~S2F}. We apply all methods from Section \ref{sec: synthetic} except for Jacobi, due to its prohibitive runtime. 
The top patterns from each method are plotted in \textbf{Figure 2F} and \textbf{Supplementary Figure S2E}.
Only MSPM and NLS recover clear digits. 

\subsection{MCPCs across cancer types}
\subsubsection{Data setup}
The pan-cancer normalized TCGA gene expression data of \cite{hoadley2018cell} and sample annotations were downloaded from UCSC Xena \cite{goldman2020visualizing}. Primary tumor samples were selected. Cancer types with fewer than thirty samples were removed. PCA with 400 components was performed using scikit-learn. PC-PC covariance matrices per cancer type were stacked into a covariance tensor. There were 10509 samples across 30 cancer types.

\subsubsection{Determining rank for MCPCA}
We performed MCPCA on all the data and on random subsets of 9000 samples at ranks from 20 to 50. We computed three metrics:
\begin{itemize}
    \item Projection similarity: the maximum cosine similarity between the MCPC projection vectors, i.e., the rows of $A^{\dagger}$, on the downsampled and original dataset, scoring a rank by the proportion of MCPCs whose projection vectors had absolute cosine similarity $> 0.9$ between downsampled and original.
    \item Context loading similarity: the Pearson correlation between context loadings on original and downsampled datasets. We calculated the proportion of MCPCs with correlation $>0.99$ across contexts.
    \item Context relationship consistency:  the pairwise correlations between contexts using the MCPC context loadings on the full dataset. We computed the Spearman correlation between these pairwise correlations across pairs of ranks.
\end{itemize}
The rank with highest score with respect to these three metrics of stability and consistency was 30. One fixed random seed was used for initialization.

\subsubsection{Interpreting MCPCs}
We multiplied the MCPC matrix by the PCA projection matrix to obtain gene scores. For each MCPC, we took the 200 genes with highest and lowest scores, and performed gene set enrichment analysis with Gene Ontology Biological Process and Cell Marker gene sets using the GSEApy implementation of Enrichr \cite{fang2023gseapy,kuleshov2016enrichr}. We selected the sigificant (defined as FDR $<.001$) gene sets per MCPC \textbf{(Supplementary Table S2)} and applied the Gemini language model to summarize these into MCPC labels.

\subsubsection{Survival analysis}
Per patient scores of MCPC10 were obtained by applying the MCPC projection matrix~$A^{\dagger}$ across all patients.  We applied Cox proportional hazards regression using the lifelines python package \cite{Davidson-Pilon2019} with  per patient scores of MCPC10 in pancreatic adenocarcinoma on overall survival. 

\subsection{MCPCs across lung cancer patients}
\subsubsection{Data setup}

We obtained the harmonized Lung Cancer Atlas \cite{salcher2022high}. We selected cancer cells from lung adenocarcinoma donors with at least 200 cancer cells. We selected non-mitochondrial genes and computed PCA on the overall dataset with 400 components using scikit-learn. We built the covariance tensor of per-patient PC-PC covariance matrices.

\subsubsection{Interpreting MCPCs}
We applied MCPCA with five components. For each component, we computed the Spearman correlation of context loadings with stage, as well as a t-test for difference in mean context loadings by stage, performing Bonferroni correction. MCPC5 was significant. We interpreted MCPC5 with GSEApy and Enrichr on the gene scores obtained by multiplication of the MCPCs with the PC projection matrix, as above. 

\subsubsection{Comparison to PCA}
We computed the significance with respect to the tests above using the standard deviation (across cells) of each PC per patient. We performed a multivariate logistic regression analysis, predicting Stages III/IV vs. Stages I/II using the mean of PCs whose mean (across cells) was associated with survival with the lifelines package.

\subsection{MCPCs with Perturb-seq data}
\subsubsection{Data setup}
The raw, genome-wide Perturb-seq data in K562 cells was downloaded from \cite{replogle2022mapping}.  Cells with at least 2000 UMIs were selected, the data was log-normalized with a scale factor of 1000, and PCA with 100 components was performed. A smaller number of preprocessing PCs was used for scalability of the algorithm to large numbers of contexts. Gene perturbations with at least 300 cells detected were selected. The input to MCPCA was the tensor of PC-PC matrices in each gene perturbation of the dataset. There were 1872 perturbations after thresholding. MCPCA was performed at ranks 5, 10, 20, 30, 40, and 50. 

\subsubsection{Benchmark construction}
GO Biological Process, Cell Marker and GWAS catalog gene sets were downloaded from \cite{xie2021gene} and restricted to genes whose perturbation effects were in the covariance tensor. Gene-gene links were defined when genes co-occurred in gene sets. We filtered maximum gene-set size at different thresholds (5, 10 and 20) for each of these gene sets. REACTOME, CORUM, and SIGNOR genesets were downloaded from \cite{croft2010reactome,giurgiu2019corum,lo2023signor} respectively, and gene-gene links were defined as co-occurrence in any geneset.

\subsubsection{Benchmark calculation}
Following \cite{bendidi2024benchmarking}, we thresholded gene-gene similarity matrices with different feature sets at the lowest 5\% and highest 5\% similarity, and the recall for gene-gene links across the genesets using different feature sets: (i) Mean overall PC across cells per gene perturbation, (ii) Mean overall PC concatenated with MCPC context loadings, (iii)  Mean overall PC concatenated with overall PC standard deviation or variance explained.
We selected the best rank for MCPC and the best number of PCs variance explained per geneset and computed the proportion of true links recovered (the recall).

\subsection{Phylogenetic analysis}
\subsubsection{Data setup}
We reduce each species’ gene expression matrix to 500 PCs and compute the $500\times 500$ covariance matrix for each species. Stacking the covariance matrices forms the covariance tensor, of size $500\times 500\times 5$, to which we apply MCPCA.

\subsubsection{Benchmark calculation}
The MCPC matrix is $500\times r$, where $r$ is the rank. We use it to construct a dendrogram via hierarchical clustering \cite{hartigan1975clustering}. For $r \geq 9$, MCPC recovers the dendrogram consistent with the known evolutionary relationships among the five species.

\subsection{MCPCs of genres of textual production and time periods}
\subsubsection{Data setup}
Metadata for the Gutenberg data was obtained from the Huggingface Gutenberg dataset (https://huggingface.co/datasets/sedthh/gutenberg\_english). This data was intersected with the PG-19 language modeling benchmark (a cleaned version of the Project Gutenberg digital library annotated by time of publication \cite{raecompressive2019}).  

The Gutenberg and PG19 metadata was used to annotate texts by the literary form and time period they were from. The time periods were 1800-1820, 1820-1840, 1840-1860, 1860-1880, 1880-1900, 1900-1918. The texts were annotated with literary form by searching the following keywords in the Gutenberg header metadata:
\begin{itemize}
    \item[Science:] `science', `philosophy', `physics', `chemistry', `biology',
        `mathematics', `astronomy', `psychology', `logic', `ethics',
        `metaphysics', `epistemology', `natural history', `geology',
        `evolution', `scientific', `algebra', `geometry', `medicine', `anatomy', `botany', `zoology', `technology', `engineering', `political science', `economics', `sociology'
    \item[Fiction:]  `fiction', `novel', `romance', `adventure', `mystery', `detective', `gothic', `love stories', `short stories',
`fantasy', '`science fiction', `historical fiction', `thriller',
`western', `sea stories', `war stories', `humorous stories'
\end{itemize}

The contexts were form/time period pairs. Sentences containing `human' were extracted.

\subsubsection{Feature extraction}
The pre-trained Bert-base-uncased model and tokenizer was downloaded from HuggingFace \cite{devlinbert2018}. Sentences containing the word human 

For each sentence in the dataset, the token positions for the first occurrence of the word human in each sentence were selected and contextualized word embeddings (the token embeddings at the final hidden layer) were averaged across these tokens. The feature dimension of that model is 768. These were the features for each sentence. For each of the contexts, the feature x feature covariance matrices across all sentences was computed. 
 
\subsubsection{MCPCA and interpretation}
MCPCA was performed with 10 components. Per-sentence scores for each MCPC were obtained via the MCPC projection matrix. 

Per-text scores were the median score of per-sentence scores across the sentences within each text. MCPCs were annotated by manual inspection of the top sentences and works.

\subsubsection{Comparison to PCs}
To evaluate whether PCA could resolve the same factors as MCPCA, we compared PCs to MCPCs. PCA with 50 PCs was performed on features of all science sentences, as well as PCs of all fiction sentences, as well as of all sentences combined. The proportion of the top 200 sentences with respect to MCPC4 and MCPC6 scores that were contained within the top 200 sentences with respect to the scores of each PC were computed. In addition, the Spearman correlation between the variance explained in each context for each PC and the MCPCA context loadings was computed.

\section*{Acknowledgments} 
S.B. was supported in part by the Eric and Wendy Schmidt Center at the Broad Institute. 
J.P. was supported in part by a start-up grant from the University of Georgia. 
J.K. was supported in part by NSF DMS 2309782, NSF DMS 2436499, NSF CISE-IIS 2312746, DE SC0025312, and the Sloan Foundation. 
A.S. was supported in part by the Sloan Foundation.

\bibliographystyle{plain}
\bibliography{new_reference}
   
\markboth{}{}

\makeatletter
\AMS@printaddresses
\makeatother

\newpage

\setcounter{figure}{0}
\renewcommand{\thefigure}{SI-\arabic{figure}}
\setcounter{section}{0}
\renewcommand{\thesection}{SI-\arabic{section}}

\section{Details of MCPCA}
\label{sec:details}

\subsection{The input and output of MCPCA}
\label{sec:input_output_model}

MCPCA is a tool to study a set of variables measured in multiple contexts. 
It studies $p$ variables, with samples in one of $k$ contexts.
The data in context $i$ are organized into a matrix of size $n_i \times p$, denoted $X_i$, 
The list of $k$ data matrices is the input to MCPCA, together with a rank $r$. 
The first step is 
the collection of $p \times p$ sample covariance matrices $S_1, \ldots, S_k$.
To do so, the data matrices are mean-centered: the mean $\mu_i \in \RR^p$ is subtracted  from each row. 
Given mean-centered data matrix $\tilde{X}_i \in \RR^{n_i \times p}$, 
the sample covariance matrix is 
$ S_i = \frac{1}{n_i-1} \tilde{X}_i\T \tilde{X}_i $.
MCPCA writes each covariance as a shared set of vectors, multiplied by context-specific weights.

\begin{definition}
\label{def:mcpca}
The \emph{MCPCA model} on $p$ variables and $k$ contexts of rank $r$ is the set of all covariance matrices $\Sigma_1, \ldots, \Sigma_k \in \RR^{p \times p}$ with 
\begin{equation}
    \label{eqn:ABA}
    \Sigma_i = AB_i A\T, \qquad \text{for } \, i = 1, \ldots, k ,
\end{equation} 
where $A \in \RR^{p \times r}$ has unit norm columns and $B_1, \ldots, B_k \in \RR^{r \times r}$ are non-negative and diagonal. 
\end{definition}

For comparison, the PCA model of rank $r$ is the set of covariance matrices of the form $\Sigma  = VDV\T$, where $V$ has size $p \times r$ with orthonormal columns, and $D$ is diagonal of size $r \times r$.

MCPCA takes a tuple of covariance matrices and finds their closest approximation in the MCPCA  model (we discuss objective functions later).  Its output is two matrices
\[ A = \begin{bmatrix} \vline & & \vline \\ 
\fa_1 & \cdots & \fa_r \\ 
\vline & & \vline \end{bmatrix} \in \RR^{p \times r} \qquad \text{and} \qquad 
B = \begin{bmatrix} \vline & & \vline \\ 
\fb_1 & \cdots & \fb_r \\ 
\vline & & \vline \end{bmatrix} \in \RR^{k \times r}.
\] 
The columns of $A$ are the MCPCs. 
They are the axes (or factors, components, or directions) in the space of variables that explain variance across a subset of contexts.
The matrix $B$ is the context loadings. Its entry at position $(i,j)$ is the weight (or importance) of the $j$th factor in the $i$th context, i.e., the $j$th entry of the diagonal matrix $B_i$. 
Not all factors appear in all contexts, since some entries of $B$ may be zero.

\subsection{MCPCA as a tensor decomposition}
\label{sec:connection_to_tensor}

Given covariance matrices $\Sigma_1, \ldots, \Sigma_k \in \RR^{p \times p}$, the covariance tensor $T$ is the $p \times p \times k$ tensor that stacks the matrices together: its entry at position $(\alpha, \beta, i)$, for $\alpha, \beta \in [p]$ and $i \in [k]$, is the covariance of variables $\alpha$ and $\beta$ in context $i$. 
The covariance tensor is partially symmetric: its entries are unchanged under swapping the first two indices. 
The MCPCA model is a coupled decomposition (or simultaneous diagonalization) of the covariance matrix from each context.

\begin{proposition}
\label{prop:connection_to_tensor}
The MCPCA model $\Sigma_i = AB_i A\T$ for all $i = 1, \ldots, k$ is equivalent to the tensor decomposition
\begin{equation}
    \label{eqn:decomp}
    T = \sum_{j=1}^r \fa_j \otimes \fa_j \otimes \fb_j , 
\end{equation}
where $\fa_j$ is the $j$th column of $A \in \RR^{p \times r}$, and $\fb_j$ is the $j$th column of $B \in \RR^{k \times r}$.     
\end{proposition}

\begin{proof}
Each summand of the tensor decomposition in~\eqref{eqn:decomp} is a $p \times p \times k$ rank-one tensor. 
The tensor $\fa_j \otimes \fa_j \otimes \fb_j$ has $(\alpha, \beta, i)$ entry $(\fa_j)_\alpha (\fa_j)_\beta (\fb_j)_i = a_{\alpha j} a_{\beta j} b_{ij}$. 
We have $b_{ij} = (B_i)_{jj}$.
Hence the $i$th slice of $T$ has $(\alpha, \beta)$ entry  
$\sum_{j=1}^r a_{\alpha j} a_{\beta j} (B_{i})_{jj}$, which is
$(AB_i A\T)_{\alpha, \beta}.$    
\end{proof}

The tensor decomposition~\eqref{eqn:decomp} is illustrated in \textbf{Supplementary Figure S1}. 
Every partially symmetric tensor (for example, the covariance tensor $T$) has a decomposition~\eqref{eqn:decomp} for large enough rank $r$.
For an input rank $r$, MCPCA seeks the closest approximation of the covariance tensor $T$ by a rank $r$ partially symmetric tensor.
The approximation 
is computed using the tensor decomposition algorithm MSPM~\cite{wang2025multi}, modified to impose non-negativity on the matrix $B$ (since each of its entries is a variance).

\begin{figure}
\centering 
\includegraphics[width = 0.7\linewidth]{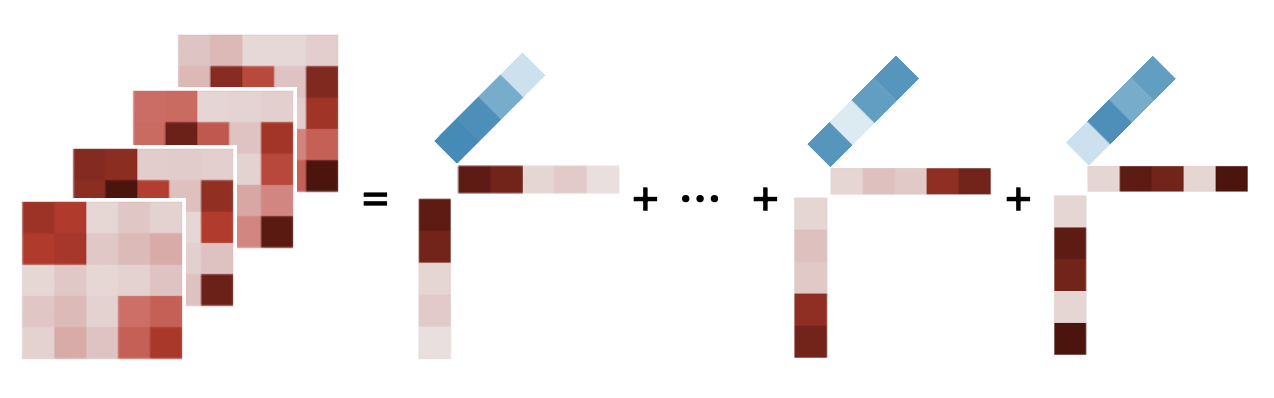}
\caption{\textbf{A tensor decomposition of the covariance tensor.} Illustrated for five variables measured across four contexts.
}
\label{SI: figure1}
\end{figure}

\subsection{MCPCA as a latent variable model}
\label{sec:uncorrelated}

In the MCPCA model, the MCPC projection transforms the data to be uncorrelated in every context. 
Let $\fx_i$ be the $p$ observed variables in context $i$. The MCPCA model is $\fx_i = A B_i^{1/2} \fz$, where $\fz$ is a shared vector of $r$ latent variables, which are uncorrelated and unit variance, the $r \times r$ matrix $B_i$ encodes their importance in context $i$, and the MCPC matrix $A$ transforms from latent to observed variables. 
Then $A^\dagger \mathbf{x}_i = B_i^{1/2} \fz$.
For comparison, PCA is the model $\fx = V \fz$, where $\fx$ are observed variables, $\fz$ are uncorrelated and unit variance latent variables, and $V$ is an orthogonal $p \times r$ matrix whose columns are the PCs. We project to the space of PCs via $V\T \mathbf{x}$.

\section{Theoretical properties of MCPCA}\label{sec:theory}

We show the identifiability of MCPCA and compute the dimension of the MCPCA model. We demonstrate how it generalizes the four main facets of usual PCA: minimizing reconstruction error, maximizing variance explained, transforming to uncorrelated variables, and maximum likelihood  estimation in a multivariate Gaussian model.

\subsection{Identifiability}
\label{sec:identifiability}

We show the uniqueness of the output of MCPCA when $r \leq p$, 
up to benign sign and ordering transformations and away from a measure zero set, as is true for usual PCA. MCPCs are unaffected by sign, just like usual PCs. A choice must be made to order the MCPCs. Usual PCs are ordered by eigenvalue. The order of weights of the MCPCs will differ across contexts, so instead the MCPCs can be ordered e.g. by the 
total variance across all contexts, i.e. by the sum of the columns of $B$.  
We show that the MCPCA model is unique away from a measure zero set after fixing the sign and order of the MCPCs.

\begin{definition}[Identifiability of latent variable models, see~\cite{allman2009identifiability}]
    A point in a statistical model with latent variables is \emph{identifiable} if its parameters can be recovered up to sign and reordering. The model is \emph{generically identifiable} if the set of 
non-identifiable parameters has measure zero in the space of parameters.
\end{definition}

 A point in the MCPCA model is a tuple of matrices $\Sigma_i = A B_i A\T$ for $i = 1, \ldots, k$. 
We show the generic identifiability of MCPCA. It does not require orthogonality of the MCPCs. 
This differs from usual PCA, which requires orthogonal PCs for uniqueness in the eigendecomposition $VDV\T$. 
Since orthogonality is not required, we do not impose it, reasoning that it may be restrictive in practice. 
The absence of orthogonality means we can in principal extend to more latent factors than observed variables (the overcomplete setting, $r > p$).  However, here we focus here on $r \leq p$. 
(When orthogonality is present, 
we investigate the benefits of MCPCA over competing tools, see Section~\ref{sec:orthogonal}.)

\begin{proposition}
\label{prop:identifiability}
   The MCPCA model is generically identifiable when $r \leq p$. 
\end{proposition}

\begin{proof}
We show the uniqueness of the tensor decomposition in Proposition~\ref{prop:connection_to_tensor}. We concatenate the covariance matrices to flatten $T$ into a matrix $M = \begin{bmatrix} \Sigma_1 & \cdots & \Sigma_k \end{bmatrix} \in \RR^{p \times pk}$.
Then $M = A (A \odot B)\T$, where $A \odot B \in \mathbb{R}^{pk \times r}$ has as columns the vectorizations of $\fa_j \otimes \fb_j$.
Since $r \leq p$, generically the row span of $M$ equals the column span of $A \odot B$.
This linear space is spanned by the rank-one matrices $\fa_j \otimes \fb_j$.  
For generic $\fa_j \otimes \fb_j$, the linear space contains only these $r$ rank-one matrices (and their scalar multiples), by the Generalized Trisecant Lemma~\cite{chiantini2002weakly}, see also \cite[Theorem 4.2 and Corollary 6.3]{wang2025multi}. 
This implies that generically $\fa_j \otimes \fa_j \otimes \fb_j$ are uniquely determined by $T$, up to scale.  
Generically, these rank-one tensors are linearly independent, which implies that their scales are also uniquely determined by $T$.  Hence the rank $r$ decomposition of $T$ is unique. 
\end{proof}

The previous result establishes that non-identifiability of MCPCA holds on a measure zero set of parameters. The next result gives a sufficient condition for identifiability at a point in the model. 

\begin{proposition}\label{prop:identifiability explicit}
    The MCPCA model is identifiable when $\fa_1, \ldots, \fa_r$ are linearly independent and no pair of $\fb_1, \ldots, \fb_r$ is linearly dependent. Linear independence $\fb_j = \lambda \fb_i$ for $i \neq j$ leads to non-identifiability. 
\end{proposition}

\begin{proof}
MCPCA is identifiable when the decomposition of $T = \sum_{j=1}^r \fa_j\otimes \fa_j \otimes \fb_j$ is unique. We use Kruskal's Theorem~\cite{kruskal1977three}, which says that the tensor decomposition of $T$ is unique when $r \leq \rm{krank}(A) + \frac12 \rm{krank}(B)- 1$, where $\rm{krank}(M)$ is the maximal number $\ell$ such that any subset of $\ell$ columns of $M$ are linearly independent.
All $\fa_j$ are linearly independent, so $\rm{krank}(A) = r$. 
No pair of $\fb_j$ is collinear, so $\rm{krank}(B) \geq 2$. 
It follows that the decomposition is unique up to $r + 1 - 1 = r$. 

 To show non-identifiability, suppose $r>1$ and that $\fb_1=\mu^2 \fb_2$ for some $\mu\neq 0$.  
Then
\[
\mu^2 \fa_1 \otimes \fa_1 + \fa_2 \otimes \fa_2
= \tfrac{\mu^2}{2}\bigl(\fa_1 + \tfrac{1}{\mu}\fa_2\bigr)^{\otimes 2}
+ \tfrac{\mu^2}{2}\bigl(\fa_1 - \tfrac{1}{\mu}\fa_2\bigr)^{\otimes 2}.
\]
Let $\fv = \fa_1 + \tfrac{1}{\mu}\fa_2$ and $\fw = \fa_1 - \tfrac{1}{\mu}\fa_2$.
We obtain an alternative decomposition
$T
= \fv^{\otimes 2} \otimes \tfrac{1}{2}\fb_1
+ \fw^{\otimes 2} \otimes \tfrac{\mu^2}{2}\fb_2
+ \sum_{j=3}^r \fa_j \otimes \fa_j \otimes \fb_j$.
It is not equivalent to the original one, since $\fv$ and $\fw$ are not scalar multiples of $\fa_1$ or $\fa_2$.
Therefore, the model is not identifiable.  
The only other case is $\fb_1 = 0$, since $B$ is nonnegative; it gives non-identifiability too. 
\end{proof}

We use identifiability to find the dimension of the MCPCA model.
    It lives in the space of $k$-tuples of $p \times p$ covariance matrices, which has dimension ${{p + 1 \choose 2}}k$, since there are ${{ p + 1\choose 2}}$ degrees of freedom in each of the $p \times p$ covariance matrices. 
    
\begin{proposition}
\label{prop:dim}
    The MCPCA model on $p$ variables and $k$ contexts of rank $r \leq p$ has dimension $r(p + k - 1)$.
\end{proposition}
\begin{proof}
    The MCPCA model says that $\Sigma_i = A B_i A\T$. The matrix $A$ has  $(p-1)r$ degrees of freedom, since it is a $p \times r$ matrix with norm one columns. The degrees of freedom in $B_i$ is~$r$, since it is diagonal, hence the total number of parameters contributed by the matrices $B_i$ is $kr$. Hence, the number of parameters for the MCPCA model is $(p-1)r + kr = r ( p + k - 1)$. The model is generically identifiable, by Proposition~\ref{prop:identifiability}. Hence, the dimension of the model equals the dimension of the parameter space.  
\end{proof}

\subsection{Reconstruction error}
\label{sec:reconstruction}

Reconstruction of the $p \times p \times k$ tensor $T$ is equivalent to reconstruction of the covariance matrices $\Sigma_i$, with respect to the Frobenius norm $\| \cdot \|_F$, which (for a matrix or a tensor) is the square root of the sum of the squares of the entries. 

\begin{proposition}
\label{prop:frob_norms}
Minimizing the average reconstruction error of the covariance matrices $ \frac{1}{k}\sum_{i=1}^k \| \Sigma_i - AB_i A\T \|_F^2$ is equivalent to finding a closest approximation of $T$ by a rank $r$ partially symmetric tensor, with respect to the Frobenius norm. 
\end{proposition}

\begin{proof}
The average reconstruction error can be rewritten as 
$\frac{1}{k} \| T - \sum_{j=1}^r \fa_j \otimes \fa_j \otimes \fb_j \|_F^2.$
\end{proof}

\subsection{Variance explained}
\label{sec:variance_explained}

\begin{definition}[Average squared variance explained by $\{\fa_1,\ldots,\fa_r\}$]
Let $\Sigma_1,\ldots,\Sigma_k \in \mathbb{R}^{p\times p}$ be covariance matrices and 
let $A = [\fa_1,\ldots,\fa_r] \in \mathbb{R}^{p\times r}$ be a matrix with unit vector columns.
Define the subspace of matrices
\[
\mathcal{\Acal} = \mathrm{span}\{  \fa_j \fa_j\T : j = 1,\ldots,r \} 
\subseteq \mathbb{R}^{p\times p}.
\]
For each $i$, let $P_{\Acal}(\Sigma_i)$ be the projection of 
$\Sigma_i$ onto the subspace $\mathcal{A}$, i.e.
\[
P_{\mathcal{A}}(\Sigma_i)
:= 
\arg\min_{M \in \Acal} \|\Sigma_i - M\|_F^2.
\]
The \emph{average squared variance explained} in $\Sigma_1,\ldots,\Sigma_k$ by 
$\{\fa_1,\ldots,\fa_r\}$ is 
$
\frac{1}{k}\sum_{i=1}^k \bigl\| P_{\Acal}(\Sigma_i) \bigr\|_F^2$.
\end{definition}

\begin{proposition}\label{prop: max var exp}
Maximizing the average squared variance explained in $\Sigma_1, \ldots, \Sigma_k$ is equivalent to computing a closest approximation of $T$ by a rank $r$ partially symmetric tensor, with respect to Frobenius norm.
\end{proposition}
\begin{proof}
Fix vectors $\fa_1,\ldots,\fa_r\in\mathbb{R}^p$ and let 
$\Acal=\mathrm{span}\{\fa_j\fa_j\T : j=1,\ldots,r\}$.    
The projection $P_{\Acal}(\Sigma_i)$ has the form 
$A B_i A\T$, where $A = [\fa_1,\ldots,\fa_r] \in \mathbb{R}^{p\times r}$ and $B_i$ is a diagonal matrix that minimizes the least squares objective 
$\|\Sigma_i - A B_i A\T\|_F^2$ (where $\Sigma_i$ and $A$ are fixed).
We have 
$\|\Sigma_i\|_F^2 = \|\Sigma_i- A B_i A\T\|_F^2 + \| AB_i A\T\|^2_F$, by orthogonality of the residual from the approximation.
Thus, maximizing the average squared variance explained is equivalent to minimizing 
$\sum_{i=1}^k \bigl\| \Sigma_i - A B_i A\T \bigr\|_F^2$, which is a closest approximation of $T$ by a rank $r$ partially symmetric tensor, with respect to Frobenius norm, by Proposition~\ref{prop:frob_norms}.
\end{proof}

\subsection{A transformation to uncorrelated variables}

Given a matrix $A$ of size $p \times r$, its pseudo-inverse is $A^\dagger$, a matrix of size $r \times p$ that satisfies conditions analogous to the inverse matrix. For our purposes, it will suffice to consider a matrix $A$ with linearly independent columns, for which $A^\dagger = (A\T A)^{-1} A\T$ and hence $A^\dagger A = I_r$, the identity matrix of size $r \times r$.
That is, $A^\dagger$ is a left inverse.
An important example is that for a matrix $V \in \RR^{p \times r}$ with orthonormal columns, $V^\dagger = V\T$. 

For usual PCA, transforming the data via $V\T$ projects to a space in which the data are uncorrelated. 
In MCPCA, transforming the variables via the MCPC projection matrix $A^\dagger$ projects to a lower-dimensional space in which the data are uncorrelated in every context. 

\begin{proposition}\label{prop: uncorrelated}
Let $\Sigma_i = A B_i A\T$ be the covariance matrix of $\fx_i$ in the $i$th context, for $i=1,\ldots,k$ where $A$ is a matrix with linearly independent columns. 
Then $A^{\dagger}$ transforms $\fx_i$ to $A^\dagger \fx_i$, whose components are uncorrelated in every context.
\end{proposition}
\begin{proof}
The covariance of $A^\dagger \fx_i$ is
\[
\Cov(A^\dagger \fx_i)
= A^\dagger \Sigma_i (A^\dagger)\T
= A^\dagger A B_i A\T (A^\dagger)\T
= B_i,
\]
which is diagonal. Hence, the components of $A^\dagger \fx_i$ are uncorrelated in each context.
\end{proof}

   Usual PCA only achieves partially uncorrelated  variables in the multi-context setting, as follows. The PCs in individual contexts transform the data to be uncorrelated only in that context. The PCs of the combined data transform it to be uncorrelated together but, after restricting to one context, the data will not remain uncorrelated in general.
   
In practice, one does not have exact fit of data to the MCPCA model.
The correlation in transformed variables can be quantified as $\sum_{i=1}^k\|\rm{offdiag}(A^{\dagger} \Sigma_i (A^{\dagger})\T)\|^2_F$, 
where $\rm{offdiag}$ sets the diagonal terms to $0$.
Minimizing this objective finds a rank $r$ tensor approximation of $T$, see \cite{ziehe2004fast}. Hence, transformation via $A^\dagger$ makes the variables approximately uncorrelated in every context. The next subsection gives a maximum likelihood estimation interpretation.

\subsection{Connection to maximum likelihood estimation}
\label{sec:mle}

MCPCA maximizes the likelihood in a multi-context Gaussian model.
This has been established for usual PCA in probabilistic PCA~\cite{tipping1999probabilistic} and for common orthogonal principal components in~\cite{flury1984common}.

\begin{definition}[{See~\cite[Section 2]{flury1984common}}]
    Fix multivariate Gaussian models $\fx_i\sim \Ncal (0, \Sigma_i)$, for $i = 1, \ldots, k$. Consider i.i.d. samples in the $i$th model, with sample covariance matrix $S_i$ and all samples independent. The \emph{common log-likelihood function} is \[ \ell_{S_1, \ldots, S_k}(\Sigma_1 , \ldots, \Sigma_k) = \sum_{i=1}^k \ell_{S_i} (\Sigma_i),\]
    where $\ell_{S_i}(\Sigma_i)$ is the usual multivariate Gaussian log-likelihood of $\Sigma_i$ given $S_i$. A maximizer of $\ell_{S_1, \ldots, S_k}$ over a model is called a \emph{maximum likelihood estimate} (MLE) given $S_1, \ldots, S_k$.
\end{definition}

\begin{definition}
    The \emph{Gaussian MCPCA} model on $p$ variables, $k$ contexts, and rank $r$ is the collection of $k$ multivariate Gaussian models $\fx_i\sim \Ncal (0, AB_i A\T)$, for $i = 1, \ldots, k$, where 
 $A$ is $p\times r$ with linearly independent columns, and the $B_i$ are $r \times r$ and diagonal.
\end{definition}

A \emph{best rank $r$ approximation} of a tensor $T$ with respect to an objective function $d$ is a minimizer of $d(T, T')$ as $T'$ varies over the set of tensors of rank (at most) $r$. In the following, we say that a tensor is a \emph{rank $r$ approximation} of $T$ if it is a best rank $r$ approximation with respect to some objective function. 

\begin{proposition}\label{prop: ml A square}
    Assume $N$ samples from each of the $k$ models in the Gaussian MCPCA model when $r = p$. 
    The MLE of the common likelihood function given 
    sample covariance matrices $S_1,\ldots,S_k$ computes a rank~$r$ approximation of the covariance tensor $T$. 
\end{proposition}
\begin{proof}
The common log-likelihood function can be written up to additive and multiplicative constants as 
\begin{equation}\label{eq: common likelihood}
    \ell_{S_1, \dots, S_k}(A, \{B_i\}_{i=1}^k) = - \sum_{i=1}^k \left( \log \det (AB_iA\T) + \tr(A^{-\mathsf{T}}B_i^{-1} A^{-1} S_i) \right) ,
\end{equation}
see \cite[Equation 2.2]{flury1984common}.
The matrix $B_i$ is a diagonal matrix that minimizes 
$$
\log\det B_i + \tr\left(B_i^{-1} A^{-1} S_i A^{-\mathsf{T}}\right),
$$
since only the $i$th summand of \eqref{eq: common likelihood} involves $B_i$ and
$
\log\det(A B_i A\T) = \log\det B_i + 2\log\det(A).
$
Let $C_i = A^{-1} S_i A^{-\mathsf{T}}$ and $B_i = \Diag(b_{i1},\dots,b_{ir})$. The objective becomes
$
\sum_{j=1}^r \left( \log b_{ij} + \frac{(C_i)_{jj}}{b_{ij}} \right),
$
which is minimized at $b_{ij} = (C_i)_{jj}$ for all $j$. Thus, we obtain 
\begin{equation}\label{eq: B_i}
B_i = \Diag\left(A^{-1} S_i A^{-\mathsf{T}}\right).
\end{equation}
Substituting \eqref{eq: B_i} into \eqref{eq: common likelihood}, 
the trace part of the common log-likelihood is constant, so the MLE minimizes 
\begin{equation}\label{eq: common likelihood reduced}
\sum_{i=1}^k \log \det (A\Diag(A^{-1} S_iA^{-\mathsf{T}})A\T).
\end{equation}
The $i$th summand is 
\begin{align*}
\log \det (A\Diag(A^{-1} S_iA^{-\mathsf{T}})A\T) &= 2\log \det A + \log \det \Diag(A^{-1} S_iA^{-\mathsf{T}})\\
&\geq 2\log \det A + \log \det A^{-1} S_iA^{-\mathsf{T}} \\
& = \log \det S_i,
\end{align*}
where the inequality follows from the Hadamard inequality, see e.g. \cite[Section 2.1]{horn2012matrix}.
Thus, \eqref{eq: common likelihood reduced} is lower bounded by $\sum_{i=1}^k \log \det S_i$, which is achieved if and only if each $A^{-1} S_iA^{-\mathsf{T}}$ is diagonal. 
Hence computing an MLE finds a matrix $A\in \RR^{p\times r}$ that approximately diagonalizes $S_1,\ldots,S_k$, which is equivalent to finding a rank~$r$ approximation of $T$.
\end{proof}

The above result is for the full-rank MCPCA model, where $r = p$. We extend it to the low-rank model, where $r < p$. 
For this purpose, we define a  \emph{common log-likelihood loss}, which 
compares the 
log-likelihood 
under unconstrained versus diagonal covariance models. 

\begin{definition}
Consider a multi-context Gaussian latent variable model $\fx_i = A \fz_i$, where $\fz_i\sim \Ncal(0,B_i)$. 
Let $S_i$ be the sample covariance matrix of samples of $\fx_i$.
Let $A^\dagger \cdot S_i = A^{\dagger} S_i (A^{\dagger})\T$.
The \emph{common log-likelihood loss}
for $\fz_1,\ldots,\fz_k$ is 
\begin{equation}\label{eq: common log-likelihood loss}
\ell^{\mathrm{loss}}_{S_1,\ldots,S_k}(A)
:=
\max_{B_i \in \RR^{r\times r}}
\ell_{A^\dagger \cdot S_1,\ldots,A^\dagger \cdot S_k}(B_1,\ldots,B_k)
\;-\;
\max_{\substack{B_i \in \RR^{r\times r} \\ \text{diagonal}}}
\ell_{A^\dagger \cdot S_1,\ldots,A^\dagger \cdot S_k}(B_1,\ldots,B_k).
\end{equation}
\end{definition}

\begin{proposition}
    Assume $\fx_i = A\fz_i$, where $\fz_i\sim \Ncal(0,B_i)$ for $i = 1, \ldots, k$. 
    Given $N$ i.i.d. samples in each model, with all $kN$ samples independent, and sample covariance matrices $S_i$ for each $\fx_i$, the matrix $A$ that minimizes the common log-likelihood loss for $\fz_1,\ldots,\fz_k$ 
    computes a rank~$r$ approximation of the covariance tensor $T$.
    \end{proposition}
\begin{proof}
First consider the unconstrained optimization.
Fixing $A$, the first term of \eqref{eq: common log-likelihood loss} is
\begin{equation}\label{eq:full-ll-uncon}
\ell^{\mathrm{uncon}}(A)
=
\max_{B_1,\ldots,B_k}
-\sum_{i=1}^k
\Bigl(
\log\det B_i
+
\tr\bigl(B_i^{-1} A^\dagger \cdot S_i \bigr)
\Bigr).
\end{equation}
The maximizer is
$B_i^{\mathrm{uncon}} = A^\dagger \cdot S_i$.
Substituting into \eqref{eq:full-ll-uncon} gives
$\ell^{\mathrm{uncon}}(A)
=
-\sum_{i=1}^k \bigl(\log\det A^\dagger \cdot S_i + p \bigr)$.
We now compute the second term of \eqref{eq: common log-likelihood loss}.
For fixed $A$, maximizing the same objective over diagonal matrices yields
$
B_i^{\mathrm{diag}} = \Diag\bigl(A^\dagger \cdot S_i\bigr),
$
and 
\begin{equation}\label{eq:ll-diag}
\ell^{\mathrm{diag}}(A)
=
-\sum_{i=1}^k\bigl(
\log\det \Diag\bigl(A^\dagger \cdot S_i \bigr) + p \bigr).
\end{equation}
We then minimize the difference 
\begin{align}
\ell_{S_1,\ldots,S_k}^{\rm{loss}}(A)  &= \ell^{\mathrm{uncon}}(A) - \ell^{\mathrm{diag}}(A)\\
&= \sum_{i=1}^k \log\det\bigl(\Diag(A^\dagger S_i(A^\dagger)\T)\bigr)- \log\det\bigl(A^\dagger S_i (A^\dagger)\T\bigr).\label{eq:gap}
\end{align}
The expression \eqref{eq:gap} is an objective function for simultaneous diagonalization of positive definite matrices \cite{pham2001joint}.
Each summand is nonnegative and vanishes if and only if
$A^\dagger \Cov(X_i)(A^\dagger)\T$ is diagonal, by the Hadamard inequality.
So minimizing $\ell_{S_1,\ldots,S_k}^{\rm loss}(A)$ is an
approximate simultaneous diagonalization of $S_1,\ldots,S_k$,
i.e. a rank~$r$ approximation of $T$.
\end{proof}

The objective function \eqref{eq:gap} is invariant under rescaling the columns of $A$ and has a statistical interpretation.
Let $\Sigma_i' = A^\dagger S_i (A^\dagger)^{\mathsf T}$ and $D_i=\Diag(\Sigma_i')$.  Then the KL divergence~\cite{bouchard2018riemannian} between the two Gaussian models is 
\[
\mathrm{KL}\left(\Ncal(0,\Sigma_i') \,\|\, \Ncal(0,D_i)\right)
=\tfrac12\bigl(\log\det D_i-\log\det\Sigma_i'\bigr).
\]
Thus, minimizing \eqref{eq:gap} is equivalent to minimizing the KL divergence between the Gaussian models
$\Ncal(0,A^\dagger S_i (A^\dagger)^{\mathsf T})$ and their diagonal approximations.
In other words, it quantifies the information lost by imposing uncorrelatedness on the latent variables $A^{\dagger} \fx_i$ in each context.

\section{Comparison to prior work}
\label{sec:comparison}

Many tools exist for multi-context data analysis. 
They use the multi-context information to find axes shared by all contexts, to remove the effect of one context on another (e.g. 	to find features in an experimental context relative to a control), or to maximize the difference between contexts. Methods include linear algebra tools to decompose a tuple of matrices~\cite{alter2003generalized,ponnapalli2011higher,khamidullina2022multilinear},
factor analysis models for finding shared and context-specific factors~\cite{klami2014group,de2019multi}, 
statistical models that combine context-dependent with context-invariant parameters~\cite{lewis2021identifying}, and methods that encode the context information as a categorical variable~\cite{friedman2009elements}.
Recently, there has been particular interest in the case of two contexts, a foreground dataset, representing an experimental group, and a background dataset, representing a control group~\cite{abid2018exploring,de2025identifying,li2020probabilistic,wang2024contrastive,zou2013contrastive}.
By comparison to previous work, MCPCA finds factors present across subsets of contexts, without optimizing objectives such as high importance in all contexts or high relative importance between a pair of contexts. Unlike previous work, it uses the importance of the shared factors in a context as an encoding of the context. 
We compare MCPCA to existing methods, from a conceptual and algorithmic point of view. 
 
\subsection{PCA}

PCA is used in two main ways for multi-context data~\cite{baharav2025stacked}. The first is to find PCs in individual contexts and then compare them. The second is to compute PCs of the combined data and then to quantify their importance in each context.

The PCs in individual contexts are obtained by decomposing the $i$th covariance matrix as $\Sigma_i = V_i D_i V_i\T$. Both the matrix of PCs and their corresponding eigenvalues vary between contexts. To decide when PCs are shared between contexts, a threshold similarity decides when two PCs are sufficiently close. This requires a hyperparameter or hypothesis test to quantify similarity, and may miss components not strong enough in any individual factor.  

Let the covariance matrix of all data across all contexts be $\Sigma_{\rm all}$. This is a weighted combination of the $\Sigma_j$ (their average if the datasets from the different contexts are all the same size). The PCA decomposition is $\Sigma_{\rm all} = V  D V\T$. One can assess the importance of a PC in context $i$ by computing the variance along direction $\fv_j$ in context $i$; this is $\fv_j\T \Sigma_i \fv_j$. 
Restricting the variance computations to individual contexts gives the importance of the PCs in each context, but it does not map the contexts into spaces where the PCs are uncorrelated (they are only uncorrelated for the data as a whole). Moreover, these PCs do not maximize variance or minimize reconstruction error in an individual context, only in the overall data. For uniqueness, usual PCA requires orthogonality of factors.

\subsection{The hypothesis of common principal components}

MCPCA builds on and generalizes Flury's work on common principal components (CPCs)~\cite{flury1983some}. In the CPC model, each covariance matrix is assumed to have the form
\begin{equation}
    \label{eqn:cpc}
    \Sigma_i = V D_i V\T,
\end{equation}
where $V$ is a square orthogonal matrix of eigenvectors and the $D_i$ are the context-specific eigenvalues.
MCPCs extends CPCs in two ways: it allows a low-rank model and removes the orthogonality requirement on the factors (the square orthogonal matrix $V$ in~\eqref{eqn:cpc} is replaced by a rectangular matrix $A$ with no orthogonality constraints).
Orthogonality of $V$ is a restrictive modeling assumption, equivalent to the commutativity of the matrices $\Sigma_i$~\cite{horn2012matrix}.  

MCPCA is more expressive than the CPC model, as can be quantified via a dimension computation.  The dimension of the rank $p$ MCPCA model is $p ( p + k - 1)$, see Proposition~\ref{prop:identifiability}. By comparison, the dimension of the CPC model is ${p \choose 2} + kp$, since it is parametrized by a $p \times p$ orthogonal matrix and $k$ diagonal matrices. The difference in dimension is ${{ p \choose 2}}$. 
    
\subsection{Two-context methods}\label{two-context}

MCPCA extends several two-context models.
Contrastive PCA (cPCA)~\cite{abid2018exploring} assumes
$$
    \Sigma_i = A B_i A\T, \qquad i=1,2,
$$
but imposes additional structure on $B_1$ and $B_2$.  
Let $I \subset [r]$ be a prescribed index set and $\alpha > 0$ a contrast parameter.
The cPCA assumptions can be written as
\begin{align*}
    (B_1)_{jj} &= 0, \quad j \in I, \\
    (B_2 - \alpha B_1)_{jj} &= 0, \quad j \in I^c.
\end{align*}
These constraints encode components that are active in one context and suppressed in the other; these are the contrastive components.  
In contrast, MCPCA places no such restrictions: the entries of $B_i$ are unconstrained and learned by the method.

For two data matrices $X_1, X_2 \in \mathbb{R}^{N_i \times p}$, the GSVD~\cite{alter2003generalized} produces
$$
    X_1 = U_1 C V\T, \qquad 
    X_2 = U_2 S V\T,
$$
where $U_1, U_2$ are orthogonal, $V$ is invertible, and 
$C,S$ are diagonal.
For mean-centered data, 
$$
    \Cov(X_1) = X_1\T X_1 = V C^2 V\T, 
    \qquad
    \Cov(X_2) = X_2\T X_2 = V S^2 V\T.
$$
Thus $V$ is the MCPCA matrix $A$, while $C^2$ and $S^2$ correspond to $B_1$ and $B_2$.  
The GSVD is the $k=2$ case of our model. When $k=2$, factors present in a subset of contexts are either in a single context or present in both.

Another two-context method is~\cite{wang2024contrastive}, which studies two datasets generated by 
$$\mathbf{x}_1 = A_1 \fz, \quad \fx_2 = A_1 \fz'+ A_2 \mathbf{s}$$ with $\fz$ independent latent variables and $(\fz',\mathbf{s})$ independent latent variables. 
The method uses higher-order statistical information (assuming non-Gaussianity), whereas MCPCA only uses second-order statistics and applies to Gaussian data.

\subsection{Coupled matrix decompositions}

The higher-order generalized singular value decomposition (HO GSVD)~\cite{ponnapalli2011higher} decomposes matrices $X_i \in \mathbb{R}^{N_i \times p}$ as
\begin{equation}\label{eq: hogsvd}
X_i = U_i \Sigma_i V\T , \qquad i=1,\dots,k,
\end{equation}
with a shared right basis $V$. 
The matrix $V$ is obtained from the eigenproblem
\[
S V = V \,\mathrm{diag}(\varsigma_1,\dots,\varsigma_p),
\]
where
\[
S
= \frac{1}{N(N-1)}
\sum_{1 \le i < j \le N}
\left( 
    D_{i,\pi} D_{j,\pi}^{-1} \;+\;
    D_{j,\pi} D_{i,\pi}^{-1}
\right),
\qquad 
D_{i,\pi} = X_i\T X_i.
\]
Unlike in the GSVD, the matrices $U_i$ in the HO GSVD are not constrained to be orthogonal. The output of interest is the subspace spanned by eigenvectors with eigenvalue 1, called the common subspace.
MCPCA relates to HO GSVD, as follows. If the $U_i$ were to be orthogonal, HO GSVD would be a full-rank specialization of MCPCA.  
Indeed, for mean-centered data,
\[
    \Cov(X_i)
    =
    V (\frac{1}{N_i - 1}\Sigma_i^{2}) V\T .
\]

MCPCA differs from HO GSVD in two ways.
HO GSVD is a full rank model that seeks the common subspace to the contexts, while 
MCPCA is a low-rank model, where MCPCs are only shared over subsets of contexts.
However, in principal HO GSVD can also be used to recover the matrix $A$. 
See \textbf{Figure 2} for a synthetic data experiment.
HO GSVD has been used to up to three contexts, whereas MCPCA can be applied to thousands of contexts (across which there may not be a common subspace).

The method ML-GSVD~\cite{khamidullina2022multilinear} assumes the same model as in~\eqref{eq: hogsvd} with the constraint that each $U_i$ is orthogonal.
They estimate the shared factor $V$ by stacking the data matrices into an order-three tensor of size $N \times p \times k$ and performing a CP decomposition.
Our setting differs because we do not assume that observations across contexts are paired, and thus there is no natural way to stack the data matrices into a tensor.  

\subsection{Econometric identification via heteroskedasticity}
A line of work in econometrics~\cite{sentana2001identification, rigobon2003identification, lewis2021identifying} uses second moments across time or regimes to identify latent factors or structural shocks.
Sentana and Fiorentini~\cite{sentana2001identification} consider a static factor model
\[
\mathbf{y}_t = \Lambda \mathbf{f}_t + \mathbf{u}_t,
\qquad
\rm{Var}(\mathbf{f}_t \mid \mathcal{F}_{t-1}) = D_t,
\]
where the conditional covariance matrices satisfy
\[
\Sigma_t = \Lambda D_t \Lambda\T + \Sigma_u.
\]
Rigobon~\cite{rigobon2003identification} studies structural models of the form
\[
\boldsymbol{\eta}_t = H \boldsymbol{\varepsilon}_t,
\qquad
\rm{Var}(\boldsymbol{\varepsilon}_t \mid s) = \Omega^{(s)},
\]
where $s\in \{1,2\}$ indexes one of two volatility regimes and $\Omega^{(s)}$ is diagonal.
Lewis~\cite{lewis2021identifying} departs from regime-based identification and shows that 
$H\otimes H$ can be identified from autocovariances of squared reduced-form innovations \(\Cov(\boldsymbol{\eta}_t \odot \boldsymbol{\eta}_t,\; \boldsymbol{\eta}_{t-\ell} \odot \boldsymbol{\eta}_{t-\ell})\).

All three models are full rank models. The work of Rigobon~\cite{rigobon2003identification} fits into the setting of Section~\ref{two-context}.
The econometric approaches~\cite{sentana2001identification, lewis2021identifying} differ from MCPCA in how the covariance is obtained and used.
In time-series settings, covariance matrices often cannot be estimated directly, since only a single multivariate observation is available at each time point.
Consequently, covariances are treated as latent objects inferred through model assumptions, whereas MCPCA takes as input the sample covariance matrices within each context.

\subsection{Algorithms for tensor decomposition}

MCPCA computes a low-rank partially symmetric CP approximation of the covariance tensor. 
Many algorithms compute a low-rank approximation of a tensor. 
In principle, any of these can recover the MCPCs and context loadings. We compare to other tensor decomposition algorithms in \textbf{Figure 1}. 

The workhorses of tensor decomposition are nonlinear least squares (NLS) and alternating least squares (ALS)~\cite{vervliet2016tensorlab}.  These algorithms are sensitive to initialization~\cite{regalia2000higher} and the output does not guarantee a connection between the summands of a rank $r-1$ decomposition and those of a rank $r$ decomposition, requiring heuristics to cluster the summands to be robust to unknown rank.  The  true rank is not known for real-world data, so decompositions that are robust to misspecified rank will show better performance in practice.  Other tensor decomposition methods are insufficiently scalable.

\subsection{MCPCA}

Having summarized prior related work, we now turn our attention to our new method, MCPCA. 
Our implementation of MCPCA relies on MSPM \cite{wang2025multi}, an algorithm for decomposing tensors that respects their symmetries which builds upon the symmetric tensor decomposition algorithm in \cite{kileel2025subspace,kileel2021landscape}. The MSPM algorithm is appropriate for MCPCA, it returns a decomposition of the tensor of stacked covariances that respects the symmetries of the MCPCA model \eqref{eqn:decomp}. Furthermore, the algorithm allows for selecting the rank in a semi-supervised way, by observing the plot of singular values of a matrix related to the covariance tensor, analogously to PCA.

However, the MSPM algorithm may return a decomposition of the covariance tensor where $B$ has negative entries, which is not valid in the MCPCA context. To address this, we first apply MSPM to the stacked covariance tensor to obtain the matrix of MCPCs $A$ (and discard the other outputs of MSPM). 
Then, since the MCPCA model is linear in $B$, we use the matrix $A$ obtained by MSPM to calculate $B$, through solving the following non-negative least squares (NNLS) problem using \cite{lawson1995solving}:
\begin{equation}\label{eq:nnls}
\min_{B\ge 0} \sum_{i=1}^p\left\|S_i - \sum_{j=1}^r b_{ij}\fa_j  \fa_j\T\right\|^2
\end{equation}
The procedure is summarized in Algorithm~\ref{alg:mcpca}.
A python implementation of MCPCA is available at \url{https://github.com/QWE123665/MCPCA}.

\begin{algorithm}[t]
	\caption{Multi-Context Principal Context Analysis (MCPCA)}
	\label{alg:mcpca}
	\begin{algorithmic}
		\Require Data matrices $X_1 \in \RR^{N_1\times p},\ldots,X_k\in \RR^{N_k\times p}$ and rank $r$
		\Ensure Matrix of MCPCs $A$ and context loading matrix $B$.
        \State $\Sigma_i \gets {\rm Cov}(X_i)$ for $i=1,\ldots,k$
		\State $T \gets \texttt{stack}(\Sigma_1,\dots, \Sigma_p)$
		\State $(A, \tilde B) \gets \texttt{MSPM}(T)$
		\State $B\gets \text{NNLS solution of \eqref{eq:nnls}}$ (using $A$ obtained by MSPM)
		\State \Return $(A, B)$
	\end{algorithmic}
\end{algorithm}

\section{Orthogonal MCPCA}
\label{sec:orthogonal} 

In this section we specialize to orthogonal MCPCs. We first specialize Definition~\ref{def:mcpca}. 

\begin{definition}
The \emph{orthogonal MCPCA model} on $p$ variables and $k$ contexts of rank $r$ is the set of all covariance matrices $\Sigma_1, \ldots, \Sigma_k \in \RR^{p \times p}$ with 
\begin{equation}
    \label{eqn:ABA ortho}
    \Sigma_i = AB_i A\T, \qquad \text{for } \, i = 1, \ldots, k ,
\end{equation} 
where $A \in \RR^{p \times r}$ has orthonormal columns and $B_1, \ldots, B_k \in \RR^{r \times r}$ are non-negative and diagonal.
\end{definition}

Under orthogonality, the MCPCs specialize to the CPC model~\cite{flury1984common} and also to usual PCs of any individual covariance matrix. 
We conduct a synthetic experiment in the orthogonal case, to demonstrate the advantage of MCPCA over alternatives even under orthogonality. Then we explain how the theory of MCPCA simplifies for orthogonal vectors.

\subsection{Synthetic experiment for orthogonal MCPCs}\label{synthetic details}

We compare MCPCA with PCA and tensor decomposition methods using synthetic data. 
We sample the columns of $A\in \RR^{100\times 60}$ by orthogonalizing the columns of a random Gaussian matrix and generate $50$ sparse diagonal non-negative matrices $B_i\in \RR^{60\times 60}$ with density 0.2 from absolute values of $\Ncal(0,1)$.  
For each $i$, we draw samples from the multivariate Gaussian distribution $\mathcal{N}(0, AB_i A\T)$ and estimate $A$. We use \rm{Ascore} to compare the recovered MCPCs with the true directions.

We perform \(40\) independent trials, each with new \(A\) and \(\{B_i\}_{i=1}^{50}\).
For each trial, we generate \(1000\) samples per Gaussian and compare the runtime and recovery accuracy of the competing methods (see \textbf{Supplementary Figure S2B}).

In the second experiment, we fix a set of \(A\) and \(\{B_i\}_{i=1}^k\), and vary the number of samples per Gaussian from \(10\) to \(10^5\).
The simultaneous diagonalization method Jacobi is two orders of magnitude slower than the other methods and is therefore omitted from this experiment (see \textbf{Supplementary Figure S2C}).

These experiments show that MCPCA is among the fastest and most accurate methods.
Empirically, the error of MCPCA measured by \(\log(1-\mathrm{Ascore})\) is $\mathcal{O}(1/N)$, where $N$ is the number of samples.
Interestingly, while the theoretical error analysis in Section \ref{sec:error} predicts a rate of $\mathcal{O}(1/\sqrt{N})$, consistent with classical covariance matrix estimation, the empirical results suggest even faster convergence in practice.
The only method with comparable performance and sample complexity is nonlinear least squares (NLS) with SVD initialization; however, NLS is less stable in the non-orthogonal setting. 

\begin{figure}
    \centering
    \includegraphics[width=0.9\linewidth]{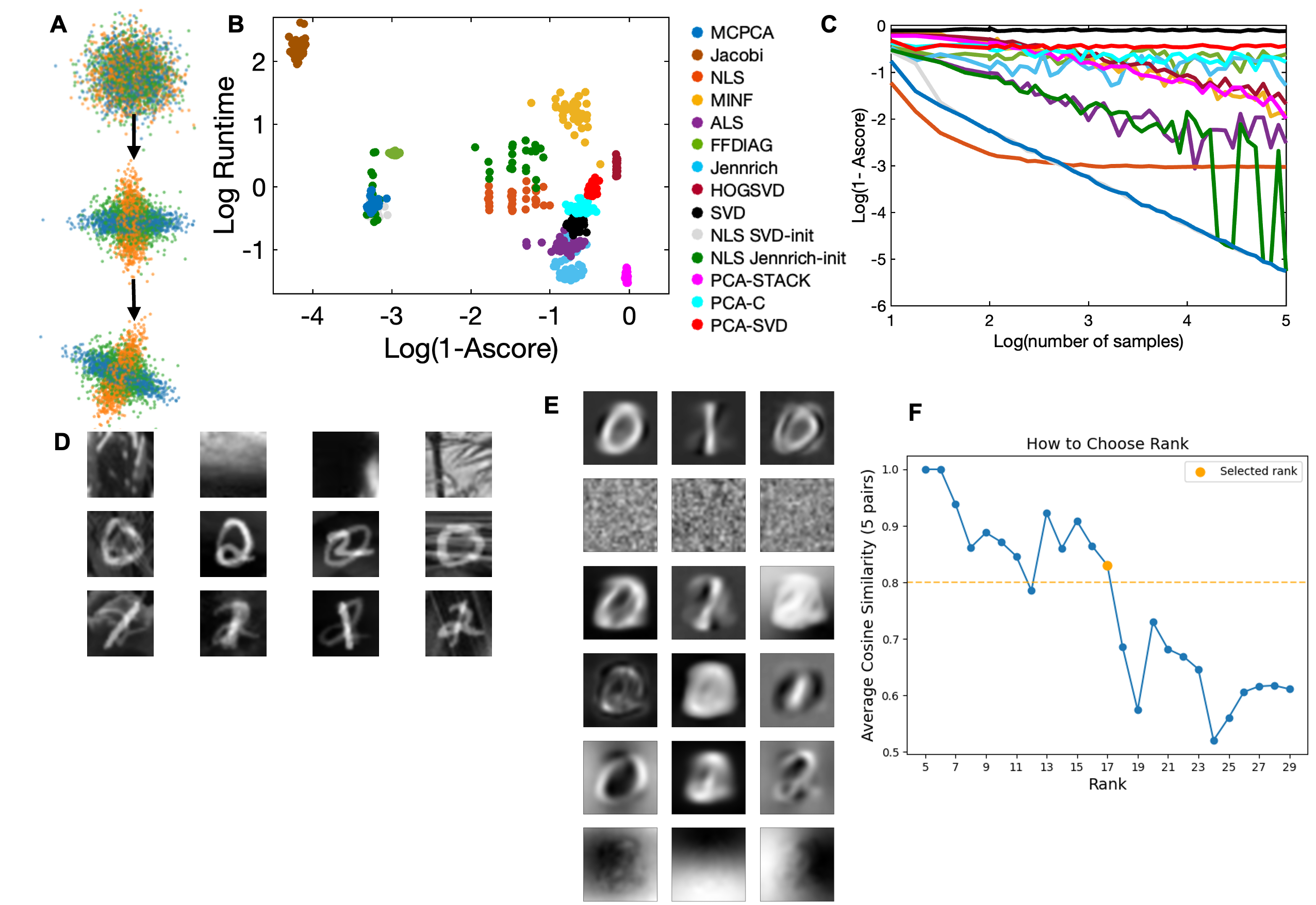}
    \captionsetup{width=0.95\textwidth}
    \caption{\textbf{⁠⁠Benchmarking MCPCA with synthetic and semi-synthetic data.}
\newline \textbf{A.} Synthetic data generated from a standard multivariate Gaussian by transforming by a context-dependent diagonal scaling followed by a shared orthogonal transformation.
\newline \textbf{B.} Runtime (in seconds) against accuracy (Ascore) for the 13 methods in the legend. MCPCA lower left, the most accurate computationally-feasible method.
\newline \textbf{C.} Accuracy compared to sample size for the 13 methods in the legend. 
\newline \textbf{D.} Examples of data from the three contexts. Top row is context 1 – background images only, second row is context 2 – zeros and twos, third row is context 3 – ones and twos.
\newline \textbf{E.} Digits recovered via other methods: Jennrich, MINF, ALS, SVD, PCA-C, PCA-SVD.
\newline \textbf{F.} How to choose the rank for the MNIST semi-synthetic experiment.
}
    \label{SIfig2}
\end{figure}

\subsection{Explanation of numerical advantages}

PCA is considered up to an elbow point. There are two reasons for this. The first is that the PCs above the elbow point explain the most variance in the data. The second, is that once the eigenvalues are less spread out the sample complexity of PCA becomes prohibitive: many samples are needed to distinguish the eigenspaces of two similar eigenvalues. The tensor decomposition approach is less affected by this phenomenon, as explained by the following result.

\begin{proposition}\label{prop:identifiability explicit orthogonal}
Suppose $\fa_1,\ldots,\fa_r$ are orthonormal.  
The MCPCA model is identifiable if and only if the vectors
$\fb_1,\ldots,\fb_r$ are pairwise distinct.
\end{proposition}
\begin{proof}
Assume that $\fb_1,\ldots,\fb_r$ are pairwise distinct.
Then there exists a vector $\fv$ such that the scalars
$\langle \fb_1,\fv\rangle,\ldots,\langle \fb_r,\fv\rangle$
are pairwise distinct.
Let $\Sigma_1,\ldots,\Sigma_k$ denote the covariance matrices with covariance tensor $T$.
Contracting $T$ with $\fv$ along the third mode yields
\[
\sum_{i=1}^k v_i\Sigma_i 
= \sum_{j=1}^r \langle \fb_j,\fv\rangle \,\fa_j \otimes \fa_j .
\]
This is an eigendecomposition, 
since the 
$\{\fa_j\}$ are orthonormal. It is unique, since the 
values $\langle \fb_j,\fv\rangle$ are distinct.
Hence the vectors $\fa_j$ are unique (up to sign and
ordering).
Once $\{\fa_j\}$ are found, the vectors $\fb_j$ are recovered uniquely
by solving a linear system.
Therefore, the decomposition is unique, and the model is
identifiable.

Conversely, suppose that $\fb_1=\fb_2$.
Then
\[
\fa_1\otimes \fa_1 \otimes \fb_1 + \fa_2\otimes \fa_2 \otimes \fb_2
= \left(\tfrac{\fa_1+\fa_2}{\sqrt{2}}\right)^{\otimes 2}\otimes \fb_1
+ \left(\tfrac{\fa_1-\fa_2}{\sqrt{2}}\right)^{\otimes 2}\otimes \fb_1 .
\]
This is an alternate decomposition of $T$ that is not
equivalent to the original one.
Hence the decomposition is not unique, and the MCPCA model is not
identifiable.
\end{proof}

The above result gives intuition behind the superiority of tensor methods over matrix methods: having the same context loading across all contexts is rarer than having the same context loading (i.e. same eigenvalue) in just one context. 
We quantify this difference by computing the dimensions of the ill-posed sets for both methods. The ill-posed set is the set of matrices where the method fails to recover a unique set of $A,\{B_i\}_{i=1}^k$ up to permutation. 
 Codimension is the dimension of the ambient space minus the dimension of the set.

\begin{proposition}
The ill-posed loci of the PCA method is codimension one; that of the tensor decomposition when $A$ has orthogonal columns is codimension $k$. 
\end{proposition}
\begin{proof}
PCA methods use eigendecomposition, which fails to be identifiable when two eigenvalues coincide. 
The condition that two eigenvalues coincide is defined by one linear equation in the entries of the $\{B_i\}_{i=1}^k$. Thus, the ill-posed set is a union of hyperplanes in the ambient space and has codimension one. 

For the tensor decomposition used in MCPCA to be identifiable, it is necessary and sufficient that the vectors $\fb_1,\ldots,\fb_r$ are distinct, see Proposition \ref{prop:identifiability explicit orthogonal}. 
The condition that $\fb_i=\fb_j$
are identical is a codimension $k$ set. The ill-posed set for tensor decomposition is thus a union of codimenison $k$ sets, which also has codimension $k$.
\end{proof}

The condition number of an identifiable parameter recovery problem is inversely proportional to the distance to its ill-posed locus~\cite{demmel1987condition}.
Since the ill-posed locus for tensor decomposition has higher codimension than that of PCA methods, a generic instance should lie farther away from degeneracy.
Thus, the tensor-decomposition problem is expected to have a smaller condition number and therefore be more stable under perturbations of the input, such as those arising from finite sample size.

\subsection{Theory of orthogonal MCPCA}

Several results in Section \ref{sec:theory} simplify when we impose orthogonality on the columns of $A$. 
The pseudoinverse of $A$ becomes $A\T$. The average squared variance explained by $A$ in Proposition \ref{prop: max var exp} simplifies to
$$\frac{1}{k}\sum_{i=1}^k \|P_\Acal(\Sigma_i)\| =\frac{1}{k} \sum_{ i=1}^k \sum_{j=1}^r \|\fa_j\T \Sigma_i \fa_j\|^2,$$
which coincides with the usual notion of variance explained by the directions $\fa_1,\ldots,\fa_r.$
The MLE in Proposition \ref{prop: ml A square} is the solution to the set of equations
$$\fa_j (\sum_{i=1}^k\frac{B_{i\ell}- B_{ij}}{B_{i\ell} B_{ij}} \Sigma_i) \fa_\ell = 0, \quad j,\ell\in [r], j\neq \ell,$$ for all $i\in [k], $ which is studied in~\cite{flury1984common,flury1986algorithm}.

\section{Error analysis}
\label{sec:error}

We quantify the error in MCPCA when the covariance tensors are estimated from finite samples. We show that
the 2-norm of the difference between the estimated 
compositions and true population factors
is $\Ocal (\sqrt{pk/N})$, where $p$ is the number of variables, $k$ the number of contexts, and $N$ the number of samples per context.

We study how noise in the covariance estimates due to finite sample size propagates to MCPCA output.
MCPCA takes as input the covariance tensor $\hat T$ whose slices are sample covariance matrices, and returns factors
$\{(\hat\fa_i,\hat\fb_i)\}$ from its rank~$r$ approximation.
We compare these to the population factors of the true covariance tensor 
and derive a bound on the cosine similarity of the estimated and true factors, which depends on sample size.
The main result is Theorem~\ref{thm: error analysis}, which shows that, under mild conditions on the pairwise cosine similarities of $\{\fa_i\}_{j=1}^r,\{\fb_i\}_{j=1}^r$, the recovered directions converge at rate $\Ocal(1/\sqrt{N})$.

Throughout this section, $\SB^{n-1}$ denotes the unit sphere in $\RR^n$, and $\Vect(\cdot)$ is the vectorization of a matrix or tensor obtained by lexicographical stacking of its entries. 

\begin{definition}[Contraction]
Let $T\in \mathbb{R}^{m_1 \times m_2 \times m_3}$ be an order-$3$ tensor, with entries $T_{i j k}$.
For $\fv^{(1)}\in \mathbb{R}^{m_1}$, the contraction of $T$ with $\fv^{(1)}$ in the first mode is the matrix
\[
T(\fv^{(1)},\ast,\ast) \in \mathbb{R}^{m_2 \times m_3},
\qquad
T(\fv^{(1)},\ast,\ast)_{j k}
=
\sum_{i=1}^{m_1} T_{i j k}\,\fv^{(1)}_{i}.
\]
For $\fv^{(1)}\in \mathbb{R}^{m_1}$ and $\fv^{(2)}\in \mathbb{R}^{m_2}$, the contraction of $T$ with $\fv^{(1)}$ and $\fv^{(2)}$ in the first two modes is the vector
\[
T(\fv^{(1)}, \fv^{(2)}, \ast) \in \mathbb{R}^{m_3},
\qquad
T(\fv^{(1)},  \fv^{(2)}, \ast)_{k}
=
\sum_{i=1}^{m_1}\sum_{j=1}^{m_2} T_{i j k}\,\fv^{(1)}_{i}\,\fv^{(2)}_{j}.
\]
Contraction in other modes, or with three vectors $\fv^{(1)},\fv^{(2)},\fv^{(3)}$, is defined analogously.
\end{definition}

We now state the main result of our error analysis.  Below, we define $T \in \RR^{p \times p \times k}$ to be the covariance tensor of the true covariance matrices $\Sigma_i$ and $\hat{T} \in \RR^{p \times p \times k}$ to be the covariance tensor of the sample covariance matrices $S_i$. 
We define $M = \begin{bmatrix} \Sigma_1 & \cdots & \Sigma_k \end{bmatrix} \in \RR^{p \times pk}$ to be the flattening of $T$ and denote the $r$-th largest singular value of $M$ by $\sigma_r(M)$.

\begin{theorem}\label{thm: error analysis}
For $i=1,\ldots,k$, let $X_i \in \RR^{N\times p}$ consist of $N$ i.i.d.\ samples from the multivariate Gaussian 
$
\mathcal{N}(0, \Sigma_i)$, 
where $\Sigma_i = A B_i A\T$, where $A\in\RR^{p\times r}$ has unit-norm columns $\{\fa_i\}_{i=1}^r$ and each $B_i\in\RR^{r\times r}$ is diagonal.
Let $\{(\hat{\fa}_i,\hat{\fb}_i) \}_{i=1}^r$ be the MCPCA output from input~$\hat{T}$.
Let $\varrho:= \max \{ \max_{i \ne j} |\langle \fa_i,\fa_j\rangle|, \max_{i \ne j} \langle \fb_i,\fb_j\rangle \}$. 
Assume that
$N\gg \frac{pk}{\sigma_r(M)^2}$ and $\varrho=\mathcal{O}(p^{-1/2})$.
After a permutation and possible sign flips of the outputs $\hat{\fa}_i$, 
we have 
\[
\cos(\hat{\fa}_{i},\fa_i) = 1-\mathcal{O}\left(\kappa(M)\sqrt{\frac{pk}{N}}\right),
\]
with probability at least $1-2k \exp(-p)$, 
where $\kappa(M)$ is the condition number of $M$.

\end{theorem}

To study the effect of perturbations in the input tensor, we express MCPCA as an optimization problem.
Let $\Acal \in \RR^{pk}$ be the span of top $r$ right singular vectors of $M$ (equivalently, a subspace of $p\times k$ matrices), and let $P_\Acal$ be the orthogonal projector onto $\Acal$. 
The rank-one matrices $\{\fa_j\otimes \fb_j\}_{j=1}^r$ are the only rank-one elements in $\Acal$ up to scale under genericity assumptions on $\fa_j,\fb_j$. Recovering them is equivalent to solving 
$$
{\rm{argmax}}_{\|\fa\|=1,\|\fb\|=1} \| P_\Acal(\fa\otimes \fb)\|^2.
$$
Let $T_\Acal$ be the stack of the top $r$ right singular vectors of $M$, reshaped into a $p\times k \times r$ tensor.
We define the function $F_{\mathcal{A}} : \mathbb{S}^{p-1} \times \mathbb{S}^{k-1} \to \mathbb{R}$ to be
\[
F_{\mathcal{A}}(\fa,\fb) = \| P_{\mathcal{A}} (\fa \otimes \fb) \|^2 = \| T_{\mathcal{A}}(\fa, \fb, *) \|^2.
\]

\subsection{Local concavity}

We show that the objective function $F_{\Acal}$ is locally concave in a neighborhood of each rank-one matrix
$\fa_i \otimes \fb_i$. Hence for a perturbed input tensor
$\hat T$, the recovered rank-one matrices $\hat{\fa}_i \otimes \hat{\fb}_i$ lie in small
neighborhoods of $\fa_i \otimes \fb_i$, of size controlled by the
perturbation $\hat T - T$.
We begin by finding the Riemannian gradient and Hessian of $F_{\Acal}$.

\begin{lemma}\label{lem: riemannian gradient and hessian}
\leavevmode
Fix $\fa\in \SB^{p-1}$ and $\fb\in \SB^{k-1}$ with $\|T_\Acal(\fa,\fb,\ast)\|\neq 0$. 
\begin{enumerate}
    \item[(1)] The Riemannian gradient of \( F_{\mathcal{A}} \) is
    \[
    {\rm{grad}}\, F_{\mathcal{A}}(\fa, \fb) =
    \begin{bmatrix}
    2 \sigma T_{\mathcal{A}}(*, \fb, \fc) - 2 \sigma^2 \fa \\
    2 \sigma T_{\mathcal{A}}(\fa, *, \fc) - 2 \sigma^2\fb
    \end{bmatrix},
    \]
    where \( \fc = \frac{T_{\mathcal{A}}(\fa, \fb, *)}{\| T_{\mathcal{A}}(\fa, \fb, *) \|} \), and \( T_{\mathcal{A}}(\fa, \fb,\fc) = \sigma>0 \).

    \item[(2)] The Riemannian Hessian of \( F_{\mathcal{A}} \) is
    \[
    {\rm{Hess}}\, F_{\mathcal{A}}(\fa, \fb) = 2
    \begin{bmatrix}
    \Pi_a M_A M_A\T \Pi_a - \sigma^2 \Pi_a & \Pi_a(M_A M_B\T + \sigma T_{\mathcal{A}}(*, *, \fc))\Pi_b \\
    \Pi_b(M_B M_A\T + \sigma T_{\mathcal{A}}(*, *, \fc)\T)\Pi_a & \Pi_b M_B M_B\T \Pi_b - \sigma^2 \Pi_b
    \end{bmatrix},
    \]
    where \( M_A = T_{\mathcal{A}}(*, \fb, *) \), \( M_B = T_{\mathcal{A}}(\fa, *, *) \), and
    \[
    \Pi_a = I_p - \fa \fa\T, \quad \Pi_b = I_k - \fb \fb\T.
    \]
    \item[(3)] For any \( \fv \perp \fa \), \( \fw \perp \fb \), we have
    \[
    \frac{1}{2}
    \begin{bmatrix}
    \fv \\ \fw
    \end{bmatrix}\T
    {\rm{Hess}}\, F_{\mathcal{A}}(\fa, \fb)
    \begin{bmatrix}
    \fv \\ \fw
    \end{bmatrix}
    = \| P_{\mathcal{A}}(\fx) \|^2 - \sigma^2 \| \fx \|^2 + 2 \langle \fv \otimes \fw, P_{\mathcal{A}}(\fa \otimes \fb) \rangle,
    \]
    where $\fx = \fv \otimes \fb + \fa \otimes \fw$.
\end{enumerate}
\end{lemma}
\begin{proof}
For (1), the Euclidean gradient of $F_\Acal$ on $\RR^{p-1}$ is
$$
\frac{\partial F_\Acal(\fa,\fb)}{\partial \fa} = 2T_\Acal(*,\fb,T_\Acal(\fa,\fb,*)) = 2\sigma T_\Acal(*,\fb,\fc).
$$
Projecting the Euclidean gradient on $\{\fa\}^\perp$, we obtain $\nabla_{\SB^{p-1}}F_\Acal(\fa,\fb) = 2\sigma T_\Acal(*,\fb,\fc)- 2\sigma^2 \fa$.  Analogously, there is a similar equation for~$\nabla_{\SB^{k-1}}F_\Acal(\fa,\fb)$. 

For (2), the Euclidean Hessian of $F_\Acal$ on $\RR^{p-1}$ is
$$
\frac{\partial^2 F_\Acal(\fa,\fb)}{\partial \fa^2} = 2T_\Acal(*,\fb,*)T_\Acal(*,\fb,*)\T= 2M_A M_A\T.
$$
We also have
\begin{align*}
\frac{\partial^2 F_\Acal(\fa,\fb)}{\partial \fa\partial \fb} &= 2T_\Acal(*,\fb,*)T_\Acal(\fa,*,*)\T 
= 2M_A M_B\T + 2T_\Acal(*,*,T_\Acal(\fa,\fb,*)) 
\\ &= 2M_A M_B\T+ 2\sigma T_\Acal(*,*,\fc).
\end{align*}
It follows that the top left block of the Riemannian Hessian is
$$
\nabla_{\SB^{n-1}}^2F_\Acal(\fa,\fb) = \Pi_a(2M_A M_A\T- \langle \frac{\partial F_\Acal(\fa,\fb)}{\partial \fa}, \fa\rangle )\Pi_a = 2(\Pi_aM_A M_A\T\Pi_a-\sigma^2\Pi_a).
$$
Analogously, the bottom right block is $\nabla_{\SB^{k-1}}^2F_\Acal(\fa,\fb)=2(\Pi_bM_B M_B\T\Pi_b-\sigma^2\Pi_b).$ Finally, the top right block (which is the transpose of the bottom left block) is
$$
\nabla_{\SB^{k-1}}\nabla_{\SB^{n-1}}F_\Acal(\fa,\fb) = \Pi_a(2M_A M_B\T + 2\sigma T_\Acal(*,*,\fc) )\Pi_b. 
$$

For (3), we expand the quadratic form using the block expression in (2):
\begin{align*}
&\frac{1}{2}
\begin{bmatrix}
\fv \\ \fw
\end{bmatrix}\T
\rm{Hess}\, F_{\mathcal{A}}(\fa, \fb)
\begin{bmatrix}
\fv \\ \fw
\end{bmatrix}\\
=&
\fv\T M_A M_A\T \fv - \sigma^2 \|\fv\|^2 + 2 \fv\T M_A M_B\T \fw + 2 \sigma T_{\mathcal{A}}(\fv, \fw, \fc) + \fw\T M_B M_B\T \fw - \sigma^2 \|\fw\|^2\\
=& \|M_A\T\fv + M_B\T \fw\|^2 +2 \sigma T_{\mathcal{A}}(\fv, \fw, \fc) - \sigma^2 (\|\fv\|^2 + \|\fw\|^2)
\\
=&\| T_{\mathcal{A}}(\fv, \fb, *) + T_{\mathcal{A}}(\fa, \fw, *) \|^2+ 2 \sigma T_{\mathcal{A}}(\fv, \fw, \fc)- \sigma^2 (\|\fv\|^2 + \|\fw\|^2).
\end{align*}
Let $\fx$ be the vectorization of the matrix \(\fv \otimes \fb + \fa \otimes \fw \), then
\[
\| T_{\mathcal{A}}(\fv, \fb, *) + T_{\mathcal{A}}(\fa, \fw, *) \|^2 = \fx\T P_{\mathcal{A}} \fx = \| P_{\mathcal{A}}(\fx) \|^2.
\]
The norm of $\fx$ satisfies $\|\fx\|^2=\|\fv\|^2+\|\fw\|^2$,
since $\fv\perp\fa$ and $\fw\perp\fb$ implies 
$\langle \fv\otimes \fb,\fa\otimes \fw\rangle=\langle \fa,\fv \rangle \langle \fw,\fb\rangle=0$.
The mixed term $2\sigma\,T_{\mathcal A}(\fv,\fw,\fc)$ is an inner product with the projection onto $\Acal$, i.e. 
$
2\sigma\,T_{\mathcal A}(\fv,\fw,\fc)
= 2\langle \fv\otimes \fw,\; P_{\mathcal A}(\fa\otimes \fb)\rangle,
$
since $\sigma\fc = T_{\mathcal A}(\fa,\fb,*)$ and $T_{\mathcal A}$ is a stack of an orthonormal basis of $\mathcal A$.  So, $T_\Acal(\fv,\fw,T_\Acal(\fa,\fb,*)) = \langle \fv \otimes \fw, P_{\mathcal{A}}(\fa \otimes \fb) \rangle$. 
Combining these identities yields
\[
\begin{bmatrix}
\fv \\ \fw
\end{bmatrix}\T
\nabla^2_{\mathbb{S}^{n-1} \times \mathbb{S}^{k-1}} F_{\mathcal{A}}(\fa, \fb)
\begin{bmatrix}
\fv \\ \fw
\end{bmatrix}
= 2 \| P_{\mathcal{A}}(\fx) \|^2 - 2 \sigma^2 \|\fx\|^2 + 4 \langle \fv \otimes \fw, P_{\mathcal{A}}(\fa \otimes \fb) \rangle. \qedhere
\]
\end{proof}

We show the function $F_\Acal$ is concave at each $(\fa_i,\fb_i)$ for $i=1,\ldots,r$ and derive a bound on the operator norm of the Riemannian
Hessian at these points. This bound allows us to translate perturbations of the subspace $\|P_{\hat{\mathcal{A}}}-P_{\mathcal{A}}\|$ into
error bounds $\|\hat{\fa}_i-\fa_i\|$.

\begin{lemma}\label{lem: F_A, mu}
For generic $\{\fa_i\}_{i=1}^r\subseteq \SB^{p-1},\{\fb_i\}_{i=1}^r\subseteq \SB^{k-1}$ and $r\leq p$, each pair \( (\fa_i, \fb_i) \) is a global maximizer of \( F_{\mathcal{A}} \). 
For any \( \fv \perp \fa_i \) and \( \fw \perp \fb_i \), we have
\[
\begin{bmatrix} \fv \\ \fw \end{bmatrix}\T \rm{Hess}\, F_{\mathcal{A}}(\fa_i, \fb_i) \begin{bmatrix} \fv \\ \fw \end{bmatrix} < 0.
\]
In fact, the following bound holds:
\[
\begin{bmatrix} \fv \\ \fw \end{bmatrix}\T \rm{Hess}\, F_{\mathcal{A}}(\fa_i, \fb_i) \begin{bmatrix} \fv \\ \fw \end{bmatrix}
\leq \mu \left( \|\fv\|^2 + \|\fw\|^2 \right),
\]
where
$
\rho_a = \max_{i \neq j} |\langle \fa_i, \fa_j \rangle|,
\rho_b = \max_{i \neq j} |\langle \fb_i, \fb_j \rangle|,
$
and
$
\mu = \frac{4(\rho_a^2 + \rho_b^2)\left[1+(r-1)(\rho_a + \rho_b)\right]}{1-(r-1)\rho_a\rho_b}-2.
$
We have $\mu < -1$ whenever $(r-1)\varrho^2 < \tfrac{1}{17}$, where $\varrho = \max\{\rho_a,\rho_b\}$.
\end{lemma}

\begin{proof}
We have $F_{\mathcal{A}}(\fa,\fb)\leq \|\fa\otimes \fb\|^2 = 1$ for all
$\fa\in\mathbb{S}^{p-1}$ and $\fb\in\mathbb{S}^{k-1}$.
Since $\fa_i \otimes \fb_i \in \mathcal{A}$, we have
$F_{\mathcal{A}}(\fa_i,\fb_i)=1$, and therefore
$(\fa_i,\fb_i)$ is a global maximizer of $F_{\mathcal{A}}$.

From Lemma \ref{lem: riemannian gradient and hessian} (3), we have
\begin{align}
&
\begin{bmatrix} \fv \\ \fw \end{bmatrix}\T 
\rm{Hess}\, F_{\mathcal{A}}(\fa_i, \fb_i) 
\begin{bmatrix} \fv \\ \fw \end{bmatrix} 
\\= 2 & \| P_{\mathcal{A}}(\fv \otimes \fb_i + \fa_i \otimes \fw) \|^2 
- 2 \| \fv \otimes \fb_i + \fa_i \otimes \fw \|^2 
+ 4 \langle \fv \otimes \fw, P_{\mathcal{A}}(\fa_i \otimes \fb_i) \rangle. \label{eq: hess eval}
\end{align}
The second term of \eqref{eq: hess eval} equals $\|\fv\|^2 + \|\fw\|^2$ and 
the third term of \eqref{eq: hess eval} vanishes since 
\[
\langle \fv \otimes \fw, P_{\mathcal{A}}(\fa_i \otimes \fb_i) \rangle = \langle \fv \otimes \fw, \fa_i \otimes \fb_i \rangle = \langle \fv, \fa_i \rangle \langle \fw, \fb_i \rangle = 0.
\]

We now derive an upper bound on $\|P_\Acal(\fv \otimes \fb_i + \fa_i \otimes \fw)\|^2$.
First, we express $P_\Acal$ in terms of $\fa_i,\fb_i$.
Let \( \mathbf{d}_i = \Vect(\fa_i \otimes \fb_i) \in \mathbb{R}^{pk} \), and
\(
D = \begin{bmatrix} \mathbf{d}_1 & \cdots & \mathbf{d}_r \end{bmatrix} \in \mathbb{R}^{pk \times r}
\)
then $\Acal$ is the column span of $D$ and
$P_{\mathcal{A}} = D (D\T D)^{-1} D\T$.
Let $\fx = \fv \otimes \fb_i + \fa_i \otimes \fw,$ then
\[
\| P_{\mathcal{A}}(\fx) \|^2 = 
\Vect(\fx)\T D (D\T D)^{-1} D\T \Vect(\fx)\leq \| (D\T D)^{-1} \|_2  \| D\T \fx \|_2^2.
\]
The quantity $\| D\T \fx \|_2^2$ is bounded above by 
\begin{align}\label{eq:Dx}
\| D\T \fx \|_2^2 &= 
\sum_{j=1}^r \left[ \langle \fa_j, \fv \rangle \langle \fb_j, \fb_i \rangle + \langle \fa_j, \fa_i \rangle \langle \fb_j, \fw \rangle \right]^2 \\
&\leq \sum_{j =1}^r 2|\langle \fa_j, \fv \rangle \langle \fb_j, \fb_i \rangle|^2 + 2|\langle \fa_j, \fa_i \rangle \langle \fb_j, \fw \rangle| ^2 
\quad (\text{AM-GM inequality})\\
& \leq 2 \rho_b^2 \| A\T \fv \|^2 + 2 \rho_a^2 \| B\T\fw \|^2,
\end{align}
where
$
A = [\fa_1 \ \cdots\ \fa_r]\in \RR^{n \times r}, \quad B = [\fb_1 \ \cdots\ \fb_r] \in \RR^{k \times r}.
$
The operator norms of $A,B$ satisfy
\[
\| A \|_2^2 \leq 1 + (r - 1) \rho_a, \quad \| B \|_2^2 \leq 1 + (r - 1) \rho_b,
\]
by Gershgorin's circle theorem.
Substituting the bound of $\| A \|_2,\|B \|_2$ in \eqref{eq:Dx}, we obtain
\[
\| D\T \fx \|_2^2 \leq (2 \rho_a^2 + 2 \rho_b^2)(1 + (r-1)(\rho_a + \rho_b))(\|\fv\|^2 + \|\fw\|^2).
\]
We have $|\sigma_r(D\T D)-1| \leq (r-1)\rho_a \rho_b$ by Gershgorin's circle theorem.  So it follows that 
\[
\| (D\T D)^{-1} \|_2  = \frac{1}{\sigma_r(D\T D)}\leq \frac{1}{1 - (r-1)\rho_a \rho_b}.
\]
Combining the upper bounds on $\| D\T x \|_2^2$ and $\| (D\T D)^{-1} \|_2$, we obtain
\[
\| P_{\mathcal{A}}(\fx) \|^2   \leq \alpha (\|\fv\|^2 + \|\fw\|^2),
\quad \text{where   }
\alpha = \frac{2(\rho_a^2 + \rho_b^2)(1 + (r-1)(\rho_a + \rho_b))}{1 - (r-1)\rho_a \rho_b}.
\]
Putting everything together, the explicit bound on $\begin{bmatrix} \fv \\ \fw \end{bmatrix}^{\top} {\rm{Hess}}\, F_{\mathcal{A}}(\mathbf{a}_i, \mathbf{b}_i)
\begin{bmatrix} \fv \\ \fw \end{bmatrix}$ is
\begin{align}
\label{eq: hessian fa}
\begin{bmatrix} \fv \\ \fw \end{bmatrix}\T 
{\rm{Hess}}\, F_{\mathcal{A}}(\fa_i, \fb_i)
\begin{bmatrix} \fv \\ \fw \end{bmatrix}  &= 2 \| P_{\mathcal{A}}(\fv \otimes \fb_i + \fa_i \otimes \fw) \|^2 - 2 \| \fv \otimes \fb_i + \fa_i \otimes \fw \|^2
 \\ &\leq (2\alpha-2) (\|\fv\|^2 + \|\fw\|^2 ).
\end{align}
Let $\mu = 2\alpha-2$ and $\varrho = \max\{\rho_a,\rho_b\}$, then $\mu \leq \frac{8 \varrho^2 (1 + 2(r-1)\varrho)}{1 - (r-1)\varrho^2} - 2$. 
If \( (r-1)\varrho^2 < \frac{1}{17} \), then
\[
\mu 
\leq \frac{1}{2(r-1)}+  \frac{1}{\sqrt{17(r-1)}}-2 <-1. \qedhere
\]
\end{proof}

Let $\hat T\in \RR^{p \times p \times k}$ be the stack of $S_1,\ldots,S_k$ and define $\hat M = \begin{bmatrix}
    S_1& \cdots & S_k
\end{bmatrix} \in \RR^{p\times pk}$.
Let  $\hat{\mathcal{A}}$ be the span of the top $r$ right singular vectors of $\hat M$.
We show that
if the distance 
\(
\| P_{\mathcal{A}} - P_{\hat{\mathcal{A}}} \|_2 
\)
is small, 
then \( F_{\hat{\mathcal{A}}} \) is concave within a spherical cap around each \( (\fa_i, \fb_i) \).

\begin{lemma}\label{lem: local bound}
Let \( (\fa_0, \fb_0) \) be a global maximizer of \( F_{\mathcal{A}} \).
Suppose
\( (\fa, \fb) \in \mathbb{S}^{p-1} \times \mathbb{S}^{k-1} \) satisfies
$\| \fa_0 - \fa \| \leq \delta, \| \fb_0 - \fb \| \leq \delta,$
and let \( \hat{\mathcal{A}} \) be a perturbation of \( \mathcal{A} \) such that
\(
\| P_{\hat{\mathcal{A}}} - P_{\mathcal{A}} \|_2 \leq \Delta.
\)
Then for all \( \fv \perp \fa \), \( \fw \perp \fb \), we have
\[
\begin{bmatrix}
\fv \\ \fw
\end{bmatrix}\T 
{\rm{Hess}}\, F_{\hat{\mathcal{A}}}(\fa, \fb)
\begin{bmatrix}
\fv \\ \fw
\end{bmatrix}
\leq \eta (\|\fv\|^2 + \|\fw\|^2),
\]
where
$\eta = \mu+2\Delta^2+18\delta^2+8\sqrt{2}\Delta\delta+10\Delta+(4+14\sqrt{2})\delta,$
and $\mu$ is defined in Lemma \ref{lem: F_A, mu}.
For sufficiently small perturbation $\Delta$, there exists $\delta$ such that $\eta<0$, i.e. $F_{\hat\Acal}$ is strictly concave around an open neighborhood of $(\fa_0,\fb_0)$ and there is a unique critical point of $F_{\hat \Acal}$ in this open neighborhood.
\end{lemma}

\begin{proof}

The product $\begin{bmatrix}
\fv \\ \fw
\end{bmatrix}\T
\rm{Hess}\, F_{\hat{\mathcal{A}}}(\fa, \fb)
\begin{bmatrix}
\fv \\ \fw
\end{bmatrix}$ is equal to
\begin{equation}\label{eq: perturb hessian}
2\|P_{\hat{\mathcal{A}}}(\fv\otimes \fb+\fa\otimes \fw)\|^2
- 2F_{\hat{\mathcal{A}}}(\fa,\fb)\|\fv\otimes \fb+\fa\otimes \fw\|^2
+ 4\langle \fv\otimes \fw,P_{\hat{\mathcal{A}}}(\fa\otimes \fb)\rangle,
\end{equation}
by Lemma \ref{lem: riemannian gradient and hessian}. We will derive an upper bound for each term of \eqref{eq: perturb hessian}.

We start with bounding $\|P_{\hat{\mathcal{A}}}(\fv\otimes \fb+\fa\otimes \fw)\|^2$.
We project the vectors \( \fv, \fw \) to $\{\fa_0\}^\perp, \{\fb_0\}^\perp$ and define
$\fv_0 = \fv - \langle \fv, \fa_0 \rangle \fa_0, \fw_0 = \fw - \langle \fw, \fb_0 \rangle \fb_0.$
Let $\fx$ denote $\fv \otimes \fb + \fa \otimes \fw$. The following inequality holds by the triangle inequality:
\begin{align}
\| P_{\hat{\mathcal{A}}}(\fx) \| 
\leq\; &
\| (P_{\hat{\mathcal{A}}} - P_{\mathcal{A}})(\fx) \|
+ \| P_{\mathcal{A}}(\fx - \fv_0 \otimes \fb_0 - \fa_0 \otimes \fw_0) \| 
+ \| P_{\mathcal{A}}(\fv_0 \otimes \fb_0 + \fa_0 \otimes \fw_0) \| \\
\label{eq: P A perturb bound}
\leq\; &
\Delta \| \fx \|
+ \|\fx - \fv_0 \otimes \fb_0 - \fa_0 \otimes \fw_0\|
+ \| P_{\mathcal{A}}(\fv_0 \otimes \fb_0 + \fa_0 \otimes \fw_0) \| 
\end{align}
The first term of \eqref{eq: P A perturb bound} is equal to $\Delta(\sqrt{\|\fv\|^2 + \|\fw\|^2})$.
The second term of \eqref{eq: P A perturb bound} is bounded from above as follows by the triangle inequality:
\begin{align*}
&\| \fv \otimes \fb + \fa \otimes \fw - \fv_0 \otimes \fb_0 - \fa_0 \otimes \fw_0 \|_F \\
\leq\, & \|\fv\otimes \fb- \fv_0\otimes \fb\|_F
+ \|\fv_0 \otimes \fb- \fv_0 \otimes \fb_0\|_F
+ \|\fa_0\otimes \fw- \fa_0\otimes \fw_0\|_F
+ \|\fa\otimes \fw_0 - \fa_0 \otimes \fw_0\|_F\\
=\,& \|\fv-\fv_0\|+ \|\fb-\fb_0\|\|\fv_0\|
+ \|\fw-\fw_0\| + \|\fa-\fa_0\|\|\fw_0\| \\ 
\leq\, & |\langle \fv,\fa_0\rangle|
+ \|\fb-\fb_0\|\|\fv\|
+ |\langle \fw, \fb_0 \rangle|
+ \|\fa-\fa_0\|\|\fw\|\\
\leq\,& \delta \|\fv\| + \delta \|\fw\| + \delta \|\fv\| + \delta \|\fw\| 
= 2\delta (\|\fv\| + \|\fw\|).
\end{align*}
Squaring expression \eqref{eq: P A perturb bound}, we obtain
\begin{align*}
&\| P_{\hat{\mathcal{A}}}(\fv \otimes\fb+ \fa \otimes \fw) \|^2 \\
\leq\; & (\Delta\sqrt{\|\fv\|^2 + \|\fw\|^2} + 2\delta(\|\fv\|+ \|\fw\|) + \| P_{\mathcal{A}}(\fv_0 \otimes \fb_0 + \fa_0 \otimes \fw_0) \| )^2
\\
\leq\; &
\| P_{\mathcal{A}}(\fv_0 \otimes \fb_0 + \fa_0 \otimes \fw_0) \|^2 
+ \Delta^2(\|\fv\|^2 + \|\fw\|^2)
+ 4\delta^2 (\|\fv\|+ \|\fw\|)^2 \\
&\;
+ 4\delta \Delta (\|\fv\|+ \|\fw\|)\sqrt{\|\fv\|^2 + \|\fw\|^2}\\
&\;+ 2\left(\Delta \sqrt{\|\fv\|^2 + \|\fw\|^2} 
+ 2\delta (\|\fv\| + \|\fw\|)\right)
\sqrt{\|\fv_0\otimes \fb_0 + \fa_0\otimes \fw_0\| }\\
\leq\; &
\| P_{\mathcal{A}}(\fv_0 \otimes \fb_0 + \fa_0 \otimes \fw_0) \|^2 
+ (\Delta^2 + 8\delta^2 + 4\sqrt{2}\Delta\delta + 2\Delta + 4\sqrt{2}\delta)
(\|\fv\|^2 + \|\fw\|^2),
\end{align*}
where the last inequality is because $\|\fv_0\otimes \fb_0 + \fa_0\otimes \fw_0\| = \sqrt{\|\fv_0\|^2 + \|\fw_0\|^2}\leq \sqrt{\|\fv\|^2 + \|\fw\|^2}$.

Next, we bound the second term of \eqref{eq: perturb hessian}; in particular, we bound $F_{\hat{\mathcal{A}}}(\fa,\fb)$.
The difference between the function values satisfies
\begin{align*}
&\,\,\,\,\,\,\,\,\,|F_{{\mathcal{A}}}(\fa_0,\fb_0)- F_{\hat{\mathcal{A}}}(\fa,\fb)|
= \big{|}\,\|P_{\hat{\mathcal{A}}}(\fa\otimes \fb)\|^2 - \|P_{\mathcal{A}}(\fa_0\otimes \fb_0)\|^2\,\big{|}\\
& = \big{|}\,\|P_{\hat{\mathcal{A}}}(\fa\otimes \fb)\| + \|P_{\mathcal{A}}(\fa_0\otimes \fb_0)\|\,\big{|} \,\,\,
    \big{|}\,\|P_{\hat{\mathcal{A}}}(\fa\otimes \fb)\| - \|P_{\mathcal{A}}(\fa_0\otimes \fb_0)\|\,\big{|}\\
&\leq 2 \|P_{\hat{\mathcal{A}}}(\fa\otimes \fb) - P_{\mathcal{A}}(\fa_0\otimes \fb_0)\|= 2\|P_{\mathcal{A}}(\fa\otimes\fb- \fa_0\otimes \fb_0)
      + (P_{\mathcal{A}}-P_{\hat{\mathcal{A}}})(\fa_0\otimes \fb_0)\|\\
&\leq 2\|\fa\otimes\fb- \fa_0\otimes \fb_0\|
      + 2\,\|P_{\hat{\mathcal{A}}} - P_{\mathcal{A}}\|_2\leq 2\sqrt{2}\,\delta + 2\Delta,
\end{align*}
where $\|\fa\otimes\fb- \fa_0\otimes \fb_0\|\leq \sqrt{2}\delta$ follows from 
$\| \fa \otimes \fb - \fa_0 \otimes \fb_0 \|^2
= 2 - 2 \langle \fa, \fa_0 \rangle \langle \fb, \fb_0 \rangle
\leq 2\delta^2.$

Lastly, we bound the third term $\langle \fv \otimes \fw, P_{\hat{\mathcal{A}}}(\fa \otimes \fb) \rangle $ of \eqref{eq: perturb hessian}:
\begin{align*}
&\langle \fv \otimes \fw, P_{\hat{\mathcal{A}}}(\fa \otimes \fb) \rangle \\
=& \langle \fv\otimes \fw - \fv_0\otimes \fw_0, P_{\hat{\mathcal{A}}}(\fa\otimes \fb)\rangle + \langle \fv_0\otimes \fw_0, P_{\hat{\mathcal{A}}}(\fa\otimes \fb)-P_{\hat{\mathcal{A}}}(\fa_i\otimes \fb_i)\rangle\\
&+ \langle \fv_0\otimes \fw_0, (P_{\hat{\mathcal{A}}}-P_{\mathcal{A}})(\fa_0\otimes \fb_0)\rangle
\\
\leq & \|\fv\otimes \fw - \fv_0\otimes \fw_0\| + \|\fv_0\|\|\fw_0\|\|\fa\otimes \fb-\fa_i\otimes \fb_i\|+\|\fv_0\|\|\fw_0\|\Delta\\
\leq & \frac{\delta^2+2\delta}{2}(\|\fv\|^2+\|\fw\|^2)+\|\fv\|\|\fw\|(\Delta+\sqrt{2}\delta)
\leq  \frac{\delta^2+2\delta}{2}(\|\fv\|^2+\|\fw\|^2)+\frac{\Delta+\sqrt{2}\delta}{2}(\|\fv\|^2+\|\fw\|^2),
\end{align*}
where the first inequality follows from triangle inequality, the last inequality folows from AM-GM inequality and we will now explain the second equality.
By the triangle inequality and the AM--GM inequality, the distance between
\( \fv \otimes \fw \) and \( \fv_0 \otimes \fw_0 \) satisfies
\begin{align*}
\| \fv \otimes \fw - \fv_0 \otimes \fw_0 \| 
& = \| \langle \fv, \fa_0\rangle \fa_0 \otimes \fw
      + \fv \otimes \langle \fw, \fb_0 \rangle \fb_0
      + \langle \fv, \fa_0\rangle \langle \fw, \fb_0 \rangle \fa_0 \otimes \fb_0 \| \\
&\leq \| \langle \fv, \fa_0 \rangle \fa_0 \otimes \fw \| 
      + \| \fv \otimes \langle \fw, \fb_0 \rangle \fb_0 \| 
      + \| \langle \fv, \fa_0 \rangle \langle \fw, \fb_0 \rangle \fa_0 \otimes \fb_0 \| \\
&\leq \delta \|\fv\|\|\fw\|
      + \delta \|\fv\|\|\fw\|
      + \delta^2 \|\fv\|\|\fw\| \leq \frac{2\delta + \delta^2}{2} (\|\fv\|^2 + \|\fw\|^2).
\end{align*}

Putting it all together, $\begin{bmatrix}
\fv \\ \fw
\end{bmatrix}\T 
\rm{Hess}\, F_{\hat{\mathcal{A}}}(\fa, 
\fb)
\begin{bmatrix}
\fv \\ \fw
\end{bmatrix}$ is bounded above by
\begin{align*}
& 2\|P_{\hat{\mathcal{A}}}(\fv\otimes \fb+\fa\otimes \fw)\|^2
- 2F_{\hat{\mathcal{A}}}(\fa,\fb)\|\fv\otimes \fb+\fa\otimes \fw\|^2
+ 4\langle \fv\otimes \fw,P_{\hat{\mathcal{A}}}(\fa\otimes \fb)\rangle\\
\leq\; 
&2\|P_{{\mathcal{A}}}(\fv_0\otimes \fb_0+\fa_0\otimes \fw_0)\|^2
+ 2(\Delta^2 + 8\delta^2 + 4\sqrt{2}\Delta\delta + 2\Delta + 4\sqrt{2}\delta)(\|\fv\|^2+\|\fw\|^2)\\
& + 2(-1 + 2\sqrt{2}\delta + 2\Delta)(\|\fv\|^2 + \|\fw\|^2)
\\
&+ 4(\frac{\delta^2+2\delta}{2}(\|\fv\|^2+\|\fw\|^2)+\frac{\Delta+\sqrt{2}\delta}{2}(\|\fv\|^2+\|\fw\|^2))\\
=\;& (\|\fv\|^2+\|\fw\|^2)(2\Delta^2+18\delta^2+8\sqrt{2}\Delta\delta+10\Delta+(4+14\sqrt{2})\delta- 2) + 2\|P_{{\mathcal{A}}}(\fv_0\otimes \fb_0+\fa_0\otimes \fw_0)\|^2
\\
\leq \;& (\|\fv\|^2+\|\fw\|^2) \left(2\Delta^2+18\delta^2+8\sqrt{2}\Delta\delta+10\Delta+(4+14\sqrt{2})\delta+ \mu\right) = \eta (\|\fv\|^2+\|\fw\|^2),
\end{align*}
where the last line follows from
\begin{align*}
&2\|P_{{\mathcal{A}}}(\fv_0\otimes \fb_0+\fa_0\otimes \fw_0)\|^2- 2(\|\fv\|^2 + \|\fw\|)^2\\
\leq& 2\|P_{{\mathcal{A}}}(\fv_0\otimes \fb_0+\fa_0\otimes \fw_0)\|^2- 2(\|\fv_0\|^2 + \|\fw_0\|)^2\\
=& \begin{bmatrix}
    \fv_0\\ \fw_0 
\end{bmatrix}\T
\rm{Hess}\,(F_\Acal(\fa_0,\fb_0)) \begin{bmatrix}
    \fv_0\\ \fw_0 
\end{bmatrix} \leq \mu (\|\fv_0\|^2 + \|\fw_0\|^2) \leq \mu (\|\fv\|^2 + \|\fw\|^2),  \quad \text{by \eqref{eq: hessian fa}}. 
\end{align*}
The constant $\eta$ is negative for sufficiently small $\Delta$ and $\delta$, since $\mu<0$ from Lemma~\ref{lem: F_A, mu}. The function $F_{\hat\Acal}$ is thus strictly concave in an open neighborhood of $(\fa_0,\fb_0)$ and there is a unique critical point of $F_{\hat \Acal}$ in this open neighbourhood.
\end{proof}

\subsection{Error analysis}
We show how noise in the covariance propagates to the MCPCA outputs in the tensor decomposition. The main tool is Lemma \ref{lem: local bound}.

\begin{lemma}\label{lem:projection matrix bound}
Let $T$ be a tensor of format $p \times p \times k$ that is symmetric under swapping the first two indices. 
Let
$M \in \mathbb{R}^{p \times pk}$
be a flattening of $T$.
Suppose \(M\) has exactly \(r\) positive singular values, ordered as
$\sigma_1(M) \geq \cdots \geq \sigma_r(M) > 0.$
Consider a perturbed tensor \(\hat{T} \) of size $p \times p \times k$ that is symmetric under swapping the first two indices 
and denote its flattening by \(\hat{M}\in \RR^{p\times pk}\). 
Assume the perturbation satisfies
\[
\Delta_M : =\|\hat{M} - M\|_2 < \sigma_r(M).
\]
Let $\hat{\Acal}$ be the span of the leading $r$ left singular vectors of $\hat{M}\T$.
Then the projection matrices \(P_{\mathcal{A}}\) and \(P_{\hat{\mathcal{A}}}\) satisfy 
\[
\|P_{\mathcal{A}} - P_{\hat{\mathcal{A}}}\|_2 
\leq \frac{\Delta_M}{\sigma_r(M) - \Delta_M}.
\]
\end{lemma}
\begin{proof}
Same as the proof of \cite[Lemma 1]{kileel2021landscape}.
\end{proof}

\begin{proposition}\label{thm: bound on sum squared}
Let \(\{ \fa_i \}_{i=1}^r \subseteq \mathbb{S}^{p-1}\) and \(\{ \fb_i \}_{i=1}^r \subseteq \mathbb{S}^{k-1}\) be two generic collections of vectors with $r\leq p$. Define the subspace
$\mathcal{A} = \mathrm{span}\{\fa_i \otimes \fb_i : i = 1,\dots,r\},$
and let \(\hat{\mathcal{A}}\) be an \(r\)-dimensional perturbation of \(\mathcal{A}\) satisfying
$\|P_{\mathcal{A}} - P_{\hat{\mathcal{A}}}\|_2 \leq \Delta.$
Assume further that the value $\varrho := \max\{\max_{i \ne j} |\langle \fa_i, \fa_j \rangle|, \max_{i\neq j}|\langle \fb_i, \fb_j \rangle|\}$ satisfies $(r - 1)\varrho^2 < \frac{1}{17}$.
Then for each global maximizer \((s_1 \fa_i, s_2 \fb_i)\) of \(F_{\mathcal{A}}\) with \(s_1, s_2 \in \{-1,1\}\), there exists a unique critical point \((\fa,\fb)\) of \(F_{\hat{\mathcal{A}}}\) in the local neighborhood
$\mathcal{S} := \{ \fx \in \mathbb{S}^{n-1} : \| \fx - s_1 \fa_i \| \leq \delta \} 
\times \{ \mathbf{y} \in \mathbb{S}^{k-1} : \| \mathbf{y} - s_2 \fb_i \| \leq \delta \},$
where $\delta$ satisfies
$2\Delta^2+18\delta^2+8\sqrt{2}\Delta\delta+10\Delta+(4+14\sqrt{2})\delta< \frac{1}{2}.$
The critical point \((\fa,\fb)\) is a local maximizer of \(F_{\hat{\mathcal{A}}}\) and satisfies
\[
\| \fa - s_1 \fa_i \|^2 + \|\fb- s_2 \fb_i \|^2 \leq 8\Delta.
\]
\end{proposition}

\begin{proof}
The function \( F_{\hat{\mathcal{A}}} \) is \( \frac{1}{2} \)-concave in \( \mathcal{S}\),
by Lemmas~\ref{lem: F_A, mu} and \ref{lem: local bound} and by the assumptions
$(r - 1)\varrho^2 < \frac{1}{17}$, and $
2\Delta^2+18\delta^2+8\sqrt{2}\Delta\delta+10\Delta+(4+14\sqrt{2})\delta< \frac{1}{2}.$
Thus, there exists exactly one local maximizer \( (\fa,\fb)\) in \( \mathcal{S}\).
We have
\[
0 \leq F_{\hat{\mathcal{A}}}(\fa, \fb) - F_{\hat{\mathcal{A}}}(s_1 \fa_i, s_2 \fb_i) 
\leq 1 - F_{\hat{\mathcal{A}}}(s_1 \fa_i, s_2 \fb_i) = \|P_\Acal(s_1\fa_i,s_2\fb_i)\|^2- \|P_{\hat{\Acal}}(s_1\fa_i,s_2\fb_i)\|^2 \leq 2 \Delta.
\]
Let \( c \) be the geodesic segment connecting \( (\fa, \fb) \) to \( (s_1 \fa_i, s_2 \fb_i) \). The function value difference of $F_{\hat \Acal}$ at \( (\fa, \fb) \) and \( (s_1 \fa_i, s_2 \fb_i) \) satisfies
\[
F_{\hat{\mathcal{A}}}(s_1 \fa_i, s_2 \fb_i) - F_{\hat{\mathcal{A}}}(\fa, \fb) \leq -\frac{1}{2} \cdot \frac{1}{2} L(c)^2,
\]
where \( L(c) \) is the length of the geodesic and the first $\frac{1}{2}$ is because  \( F_{\hat{\mathcal{A}}} \) is \( \frac{1}{2} \)-concave in \( \mathcal{S} \).
We thus obtain
\[L(c)^2 \leq 8 \Delta.\]

The geodesic length satisfies
\[
L(c) = \sqrt{ \cos^{-1}(\langle \fa, s_1 \fa_i \rangle)^2 + \cos^{-1}(\langle \fb, s_2 \fb_i \rangle)^2 }.
\]
Using the inequality \( 2 - 2\cos \theta \leq \theta^2 \), we have $\cos^{-1}(x)^2\geq 2-2x$, and we obtain
\[
L(c)^2 \geq 2 - 2 \langle \fa, s_1 \fa_i \rangle + 2 - 2 \langle \fb, s_2 \fb_i \rangle = 
\| \fa - s_1 \fa_i \|^2 + \|\fb- s_2 \fb_i \|^2.
\]
Hence,
\[
\| \fa - s_1 \fa_i \|^2 + \|\fb- s_2 \fb_i \|^2 \leq 8\Delta.\qedhere
\]
\end{proof}

We have shown that there exists a local maximizer of $F_{\hat{\Acal}}$ in a neighborhood of one of the pairs $(\fa_i,\fb_i)$. However, $F_{\hat{\Acal}}$ may also admit other local maximizers that are far from the ground-truth vectors.
We show that any critical point of $F_{\hat{\Acal}}$ that is far from the ground truth attains a low objective value.

\begin{lemma}\label{lem: high function value implies close to global maxima}
Let \(\{ \fa_i \}_{i=1}^r \subseteq \mathbb{S}^{p-1}\) and \(\{ \fb_i \}_{i=1}^r \subseteq \mathbb{S}^{k-1}\) be two generic collections of vectors with $r\leq p$. 
Define the subspace
$\mathcal{A} = \mathrm{Span}\{\fa_i \otimes \fb_i : i = 1,\dots,r\},$
and let \(\hat{\mathcal{A}}\) be an \(r\)-dimensional perturbation of \(\mathcal{A}\) satisfying
$\|P_{\mathcal{A}} - P_{\hat{\mathcal{A}}}\|_2 \leq \Delta.$
If
\[
F_{\hat{\mathcal{A}}}(\fa,\fb) \ge \frac{2-\delta^2}{2}\,C + 2\Delta,
\]
where
\(
C
=
\frac{
\bigl(1+(r-1)\max_{i\neq j}|\langle \fa_i,\fa_j\rangle|\bigr)^{1/2}
\bigl(1+(r-1)\max_{i\neq j}|\langle \fb_i,\fb_j\rangle|\bigr)^{1/2}
}{
1-\max_{i\neq j}|\langle \fa_i,\fa_j\rangle\langle \fb_i,\fb_j\rangle|
},
\)
then there exists an index $i$ and signs $s_1,s_2\in\{\pm1\}$ such that
\[
\|\fa - s_1\fa_i\| \le \delta,
\qquad
\|\fb - s_2\fb_i\| \le \delta.
\]
\end{lemma}
\begin{proof}
We have $$|F_{\mathcal{A}}(\fa,\fb)-F_{\hat{\mathcal{A}}}(\fa, \fb)|\leq 
\|P_\Acal(\fa,\fb) + P_{\hat\Acal}(\fa,\fb)\| \|P_{\hat\Acal}(\fa,\fb) - P_\Acal(\fa,\fb)\| \leq 2 \Delta$$ by definition of $F_\Acal$, 
so it follows that
$
F_{\mathcal{A}}(\fa,\fb)\geq F_{\hat{\mathcal{A}}}(\fa,\fb)-2 \Delta \geq \frac{2-\delta^2}{2}C.
$
The inequality
$
F_{\mathcal{A}}(\fa, \fb)\leq C \max_j|\langle\fa_j,\fa\rangle\langle\fb,\fb_j\rangle|.
$
holds by Theorem 5.7 of \cite{wang2025multi}.
Thus, we obtain
$$
\max_j|\langle\fa_j,\fa\rangle\langle\fb,\fb_j\rangle|\geq\frac{2-\delta^2}{2} 
$$
and there exists an index $i$ such that 
$
\|\fa - \fa_i\|\leq \delta, \quad \|\hat \fb - \fb_i\|\leq \delta.
$
Hence, \[
\|\fa - s_1\fa_i\| \le \delta,
\qquad
\|\fb - s_2\fb_i\| \le \delta. \qedhere
\]
\end{proof}

\begin{proof}[Proof of Theorem \ref{thm: error analysis}]
Let $\hat M = \begin{bmatrix}
    S_1&\cdots S_k
\end{bmatrix}$ and $\hat\Acal$ be the span of the top $r$ right singular vectors of $\hat M$. Letting $\mathbf{u}\in \SB^{p-1}, \fv\in \SB^{pk-1}$, we have
\begin{align*}
\mathbf{u}\T (M-\hat M) \fv
    &= \sum_{i=1}^k \mathbf{u}\T (\Sigma_i-S_i ) \fv_i \leq \sum_{i=1}^k \|\fv_i\| \|\Sigma_i-S_i\|_2 \\
    & \leq \sum_{i=1}^k \|\fv_i\| \|\Sigma_i\|_2  \, \mathcal{O}(\sqrt{\frac{p}{N}}) \leq \sigma_1(M)\mathcal{O}(\sqrt{\frac{pk}{N}}),
\end{align*}
with probability at least $1-2k \exp(-p)$, 
see~\cite[Theorem 4.7.1]{vershynin2018high}.  
Thus, $\|M-\hat M\|_2 = \sigma_1(M)\mathcal{O}(\sqrt{\frac{pk}{N}})$.  By Lemma \ref{lem:projection matrix bound} and since $N\gg \frac{pk}{\sigma_r(M)^2}$, we have $\|P_\Acal-P_{\hat\Acal}\|_2 = \mathcal{O}(\sqrt{\frac{pk}{N}})$.
Let $\kappa(M)$ be the condition number of $M$.
Since $\rho = \mathcal{O}(p^{-\frac{1}{2}})$, the assumptions in Proposition \ref{thm: bound on sum squared} are satisfied and we obtain
\[\cos(\hat{\fa}_{\pi(i)},\fa_i) = 1- \frac{1}{2}\|\hat{\fa}_{\pi(i)}-\fa_i\|^2 = 1-\mathcal{O}(\frac{\sigma_1(M)\sqrt{pk}}{\sigma_r(M)\sqrt{N}}) = 1- \mathcal{O}(\kappa(M)\sqrt{\frac{{pk}}{{N}}}), \]
provided that the algorithm converges to a point sufficiently close to one of the ground-truth vectors.
Such convergence is guaranteed by examining the objective value at the limit point via Lemma~\ref{lem: high function value implies close to global maxima}.
\end{proof}

\setcounter{figure}{2} 

{
\renewcommand{\thefigure}{SI-3-\arabic{figure}}
\setcounter{figure}{0}
\begin{figure}[htbp]
    \centering
    \includegraphics[width=0.9\linewidth]{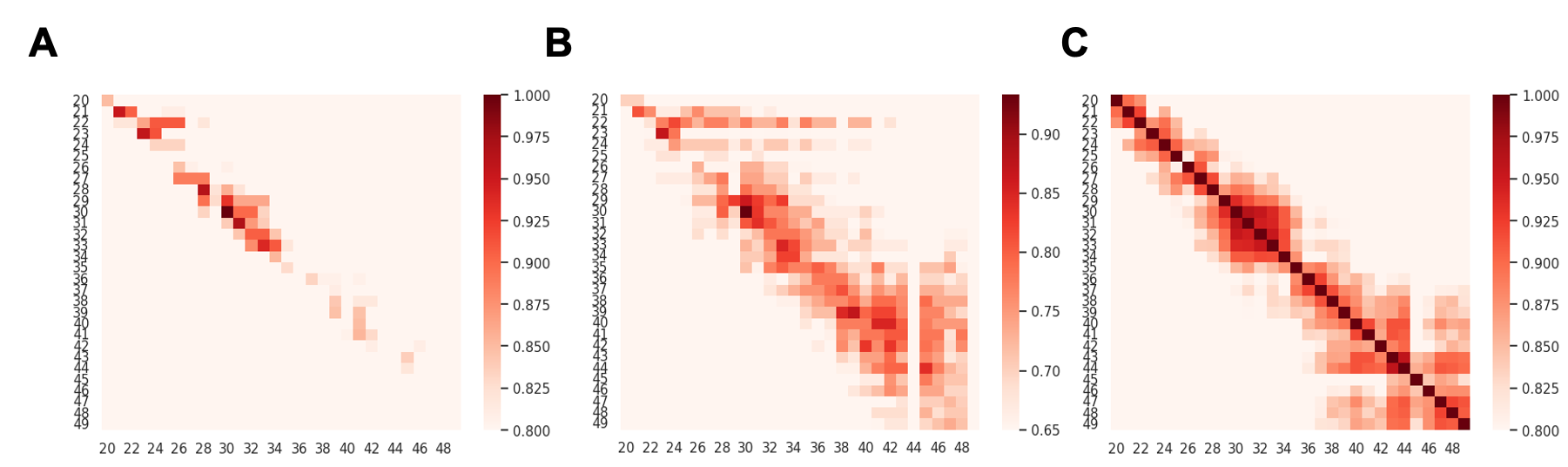}
    \captionsetup{width=0.95\textwidth}
    \caption*{FIGURE SI-3-1
    \textbf{A.} Determining rank for TCGA decomposition by comparison of proportion of MCPCs consistent between downsampled and original datasets. Color is proportion of MCPCs with correlation between projection matrices $> 0.99$ between original and downsampled dataset.
\newline \textbf{B.}
Determining rank for TCGA decomposition by comparison of proportion of MCPCs consistent between downsampled and original datasets. Color is proportion of MCPCs with correlation of context loadings $> 0.99$ between original and downsampled dataset.
\newline \textbf{C.}
Stability of context relationships by rank. Spearman correlation between pairwise correlations between context loadings at different ranks.
}
    \label{SIfig3}
\end{figure}

\renewcommand{\thefigure}{SI-3-\arabic{figure}}
\setcounter{figure}{1}
\begin{figure}[htbp]
    \centering
    \includegraphics[width=0.9\linewidth]{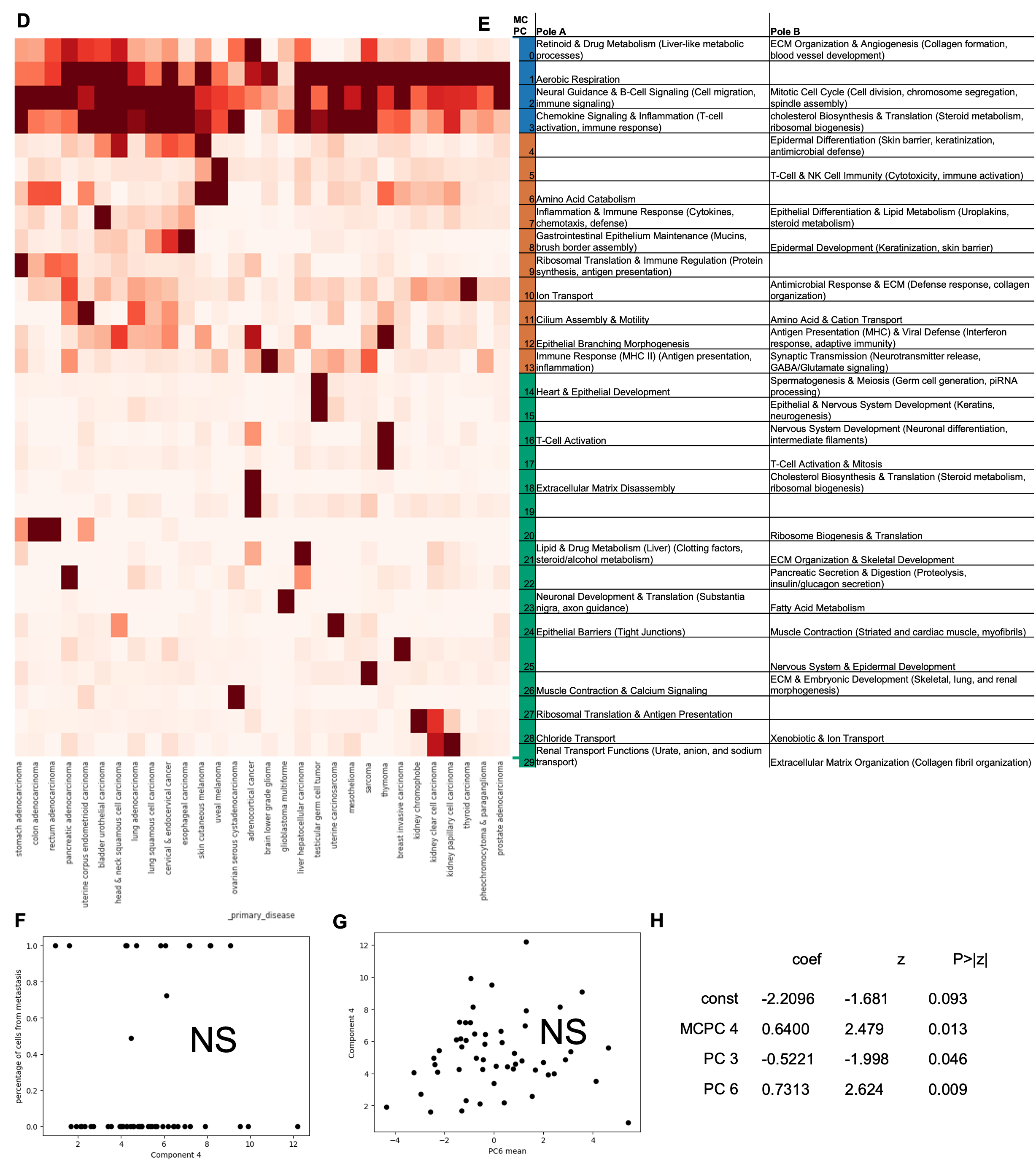}
    \captionsetup{width=0.95\textwidth}
    \caption{
\newline \textbf{D.}
Heatmap of context loadings of MCPCs across cancer types.
\newline \textbf{E.}
Annotation of context loadings and MCPCs by gene set enrichment among top 200 genes per MCPC.
\newline \textbf{F.} Correlation of MCPC5 context loading in lung cancer dataset with proportion of metastatic cells. NS indicates not significant. 
\newline \textbf{G.}
Correlation of MCPC5 context loading with mean of overall PC6.
\newline \textbf{H.} Multivariate cox proportional hazards regression results for MCPCs and overall PC means associated marginally with survival.
}
    \label{SIfig3}
\end{figure}
}

\setcounter{figure}{3} 
\begin{figure}[htbp]
\centering
    \includegraphics[width=0.84\linewidth]{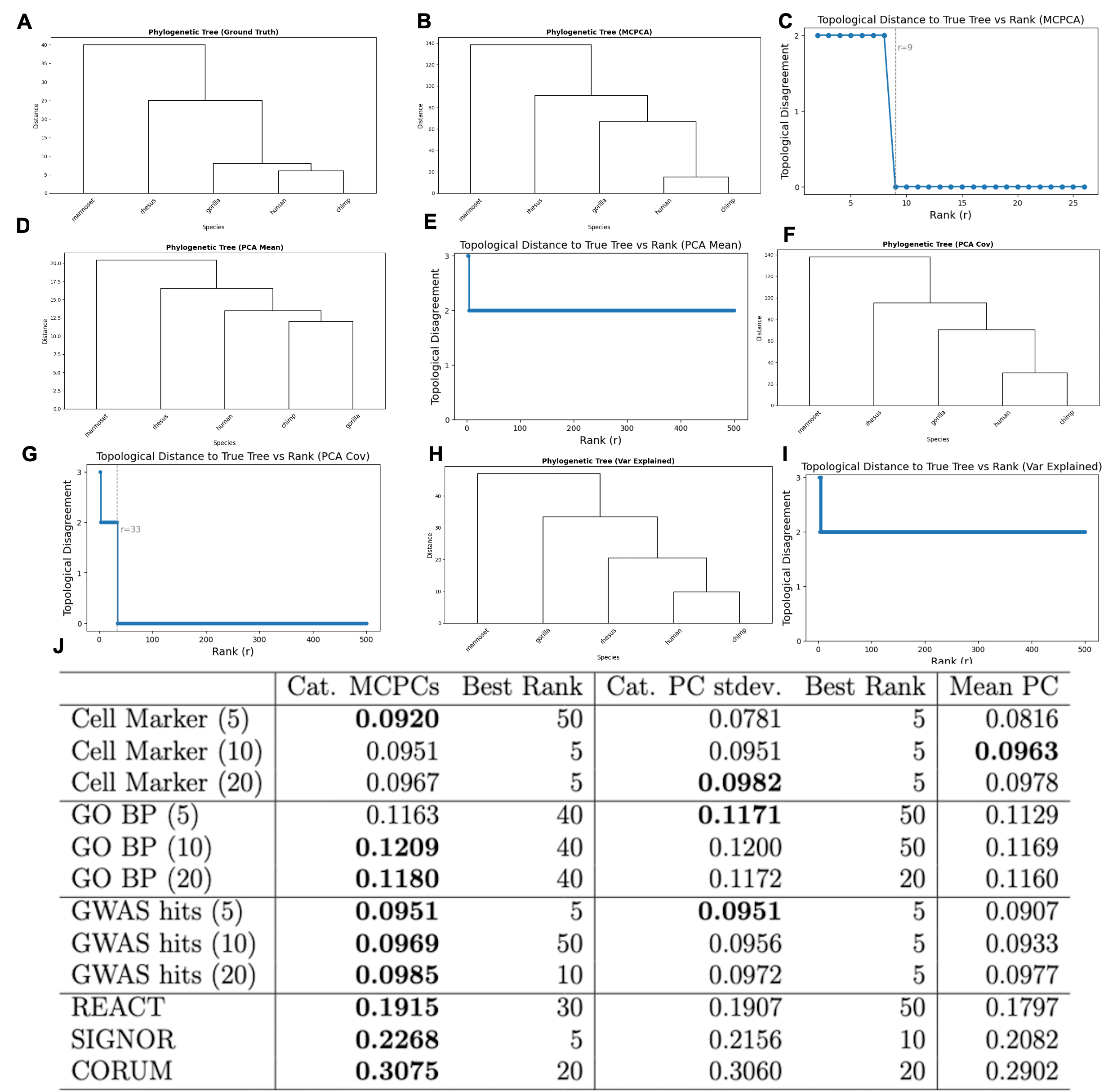}
    \captionsetup{width=0.95\textwidth}
    \caption{\textbf{A.} Ground truth phylogenetic tree. 
\newline \textbf{B.} Hierarchical clustering tree between species using the loading matrix of MCPCs across gene $\times$ gene covariance (rank 9) as (species-level) features.
\newline \textbf{C.} Edit distance to ground truth tree using MCPCA at different ranks on species $\times$ gene $\times$ gene covariance as features.
\newline \textbf{D.} Hierarchical clustering tree between species using mean PCs as features.
\newline \textbf{E.} Edit distance to ground truth tree across different numbers of PCs, using the mean PC per species as features.
\newline \textbf{F.} Hierarchical clustering tree between species using flattened covariance matrix between overall PCs1-33 as features.
\newline \textbf{G.} Edit distance to ground truth tree using flattened covariance matrix between different numbers of overall PCs ($x$-axis) as features.
\newline \textbf{H.} Hierarchical clustering tree between species using variance explained by PCs.
\newline \textbf{I.} Edit distance to ground truth tree using variance explained by PCs tree across different numbers of PCs.
\newline \textbf{J.} Benchmarking of MCPCs for gene perturbation representation learning with Perturb-seq data. Recall for gene-gene links across gene sets (rows) obtained by thresholding gene-gene similarity between gene perturbations computed with different features (columns). Cat MCPCs: concatenating MCPC context loadings on perturbation $\times$ gene $\times$ gene matrices to mean PCs in each perturbation. Cat PC stdev.: concatenating standard deviation of PCs to mean PCs for each perturbation.}
    \label{SIfig4}
\end{figure}

\end{document}